\documentclass[11pt]{article}

\usepackage{amsmath,amssymb,amsfonts,amsthm}
\usepackage{thmtools}
\usepackage[numbers,sort]{natbib}
\usepackage{xcolor}
\definecolor{dkred}  {rgb}{.7,0,0} 
\definecolor{midnight}  {rgb}{0,0,.5} 
\definecolor{forest}  {rgb}{0,.6,0} 
\usepackage[colorlinks,
            citecolor=forest,
			linkcolor=midnight,
			anchorcolor=red,
			urlcolor=midnight]{hyperref}  
\usepackage{xurl}      
\newcommand{\becca}[1]{\textcolor{dkred}{[becca: #1]}}

\usepackage[capitalize,noabbrev]{cleveref}
\Crefname{equation}{Eq.}{Eqs.}

\usepackage{microtype}
\usepackage{graphicx}
\usepackage{subcaption}
\usepackage{booktabs} 

\usepackage{appendix}
\usepackage{verbatim}

\usepackage{arxiv}

\usepackage{amsmath}
\usepackage{amssymb}
\usepackage{mathtools}
\usepackage{amsthm}

\usepackage{xcolor}
\usepackage{graphicx} 
\usepackage{amsfonts}
\usepackage{amssymb}
\usepackage{amsmath}
\usepackage{amsthm}

\usepackage{subcaption}
\usepackage{bbm, dsfont}

\newcommand{\fitparenth}[1]{\left(#1\right)}
\newcommand{\fitbracket}[1]{\left[#1\right]}
\newcommand{\inprod}[1]{\left\langle#1\right\rangle}

\newcommand{\vect}[1]{\mathbf{#1}}

\newcommand{\RR}{\mathbb{R}}

\newcommand{\CalN}{\mathcal{N}}
\newcommand{\CalS}{\mathcal{S}}
\NewDocumentCommand{\expect}{O{} m}{\mathbb{E}_{#1}\left[#2\right]}

\theoremstyle{plain}
\newtheorem{theorem}{Theorem}[section]
\newtheorem{proposition}[theorem]{Proposition}
\newtheorem{lemma}[theorem]{Lemma}

\newtheorem{definition}[theorem]{Definition}
\newtheorem{assumption}[theorem]{Assumption}
\newtheorem{remark}{Remark}

\Crefname{proposition}{Proposition}{Propositions}

\usepackage{authblk}
\usepackage[compact]{titlesec}

\author[*,1,2,3]{Andrew Dennehy}
\author[3]{Ramchandran Muthukumar}
\author[1,2,3,4,5,6]{\\Rebecca Willett}
\author[1,2,5,6]{Nisha Chandramoorthy}

\affil[1]{Committee on Computational and Applied Mathematics, University of Chicago}
\affil[2]{Department of Statistics, University of Chicago}
\affil[3]{Data Science Institute, University of Chicago}
\affil[4]{Department of Computer Science, University of Chicago}
\affil[5]{NSF-Simons National Institute for Theory and Mathematics in Biology}
\affil[6]{NSF-Simons National AI Institute for the Sky}
\affil[*]{Corresponding Author: \href{mailto:nishac@uchicago.edu}{nishac@uchicago.edu}}

\usepackage[textsize=tiny]{todonotes}

\input{macros}
\usepackage{titling}
\title{Diffusion models recover accurate mixture weights despite score function insensitivity}
\begin{document}

\maketitle

\begin{abstract}
Score-based generative models exhibit a puzzling behavior: they often appear to cover all modes of a target multimodal distribution and yet may fail to learn the correct relative mode amplitudes, which can be interpreted as mixture weights. 
We resolve this apparent paradox 
by relating the diffusion score matching (DSM) loss to the error in estimating mixture weights from generated samples. We show that, even when the target score is insensitive to mixture weights, generated samples can recover the weights accurately if the scores at intermediate noise levels are informative about the weights. Accordingly, we define the \textit{diffusion score sensitivity index} (DSSI) as the variation in the DSM loss relative to changes in a parameter. We then show that the DSSI governs the accuracy with which the parameter of the target distribution can be estimated from generated samples.
For Gaussian mixtures in arbitrary dimensions, we prove that the mixture weight estimation errors are on the same order as the DSM loss under mild conditions. Empirically, we show the emergence of sensitivity during the noising process of benchmark data distributions under typical noise schedules, and that these sensitivity values predict how well a well-trained model recovers mixture weights. Furthermore, we show that the choice of noise schedule can reduce diffusion sensitivity, leading to mode amplification. Although we focus on mixture weights, the proposed sensitivity framework governs the recovery of any qualitative parameter of the target distribution. 
 
\end{abstract}

\section{Introduction}
Generative models are powerful tools that are increasingly used in scientific research \citep{lanusse2021deep,price2023gencast,han2025learning}, medicine \citep{loni2025review,fahrner2025generative,teo2025generative}, the justice system \citep{ludwig2024machine,posner2025judge}, and more. In these and other settings, users rely on generative models to accurately reflect the qualitative properties of the true target distribution. Past work notes that learning the score of a target distribution may be insufficient for accurate sampling \citep{koehler2022statisticalefficiencyscorematching, chen2021dimensionfreelogsobolevinequalitiesmixture, li2024softmixturedenoisingexpressive}, particularly for canonical multimodal distributions such as Gaussian mixtures \citep{1100034, yakowitz1968identifiability}. 
Generative models that incorrectly reflect the amplitudes or variances of modes \citep{zhang2025collapse, roos2026met, hakemi2025deeper} may hinder uncertainty quantification and the recovery of qualitative parameters \citep{bouchet2019rare, addison2024cpmgem, mudur2023cosmological}. 
Our work proposes an explicit evaluation of generative models through the lens of parameter recovery. 
Such an evaluation augments sample-quality metrics, such as the Frechet inception distance (FID) \citep{NIPS2017_8a1d6947}, which may not explicitly test for target-specific structure, such as multimodality \citep{sajjadi2018assessing}. 

\textbf{In this work, we present an information-theoretic mechanism that determines when such key parameters may be represented accurately by diffusion models.} In the Gaussian mixture setting, even when the score matching loss is insensitive to mixture weights, learning a diffusion model anneals the score matching loss across steps of the diffusion process \citep{song2019generative}; this annealing leads to sensitivity to mixture weights, resulting in accurate estimates that were unanticipated by past analyses that focused only on the (unannealed) score matching loss \citep{koehler2022statisticalefficiencyscorematching}. 

This idea is illustrated in \cref{fig:scores}. The top row plots the score functions for two Gaussian mixture distributions that are identical except for the mixture weight $\gamma$ and shows that they are nearly identical except in a region with very low probability mass. We are thus unlikely to have sufficient training samples for the loss to reliably indicate the value of $\gamma$ that best represents the training data. However, for larger values of $t$ in the diffusion process ($t=0.25$ or $t=0.50$ in the figure), the corresponding distributions are less well separated, and the score functions are clearly distinguishable over a higher-probability region. Since training diffusion models uses the score across multiple $t \in [0,1]$, learned models hone in on the correct value of $\gamma$.

\begin{figure}[ht]
    \centering\includegraphics[width=0.7\linewidth]{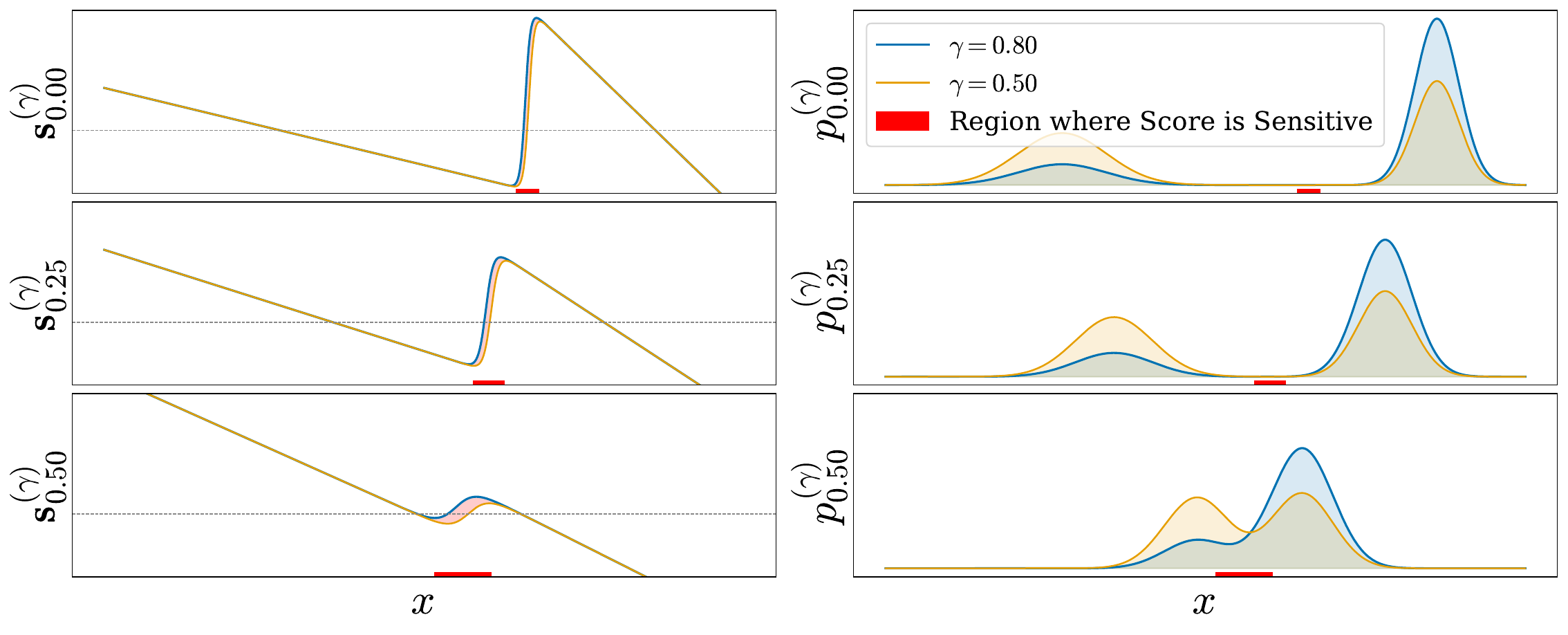}
    \caption{
    Score sensitivity to mixture weights at different diffusion times. At $t=0$, the scores for $\gamma=0.8$ and $\gamma=0.5$ are nearly identical except on a low-probability region. At $t>0$, the mixture components overlap more, and the region where the scores differ has a substantially higher probability.
    }
    \label{fig:scores}
\end{figure}

To quantify this phenomenon, we measure the sensitivity of the diffusion score matching loss $\ldsm$  \eqref{eq:diffScoreMatchingLoss} to changes in parameters of interest. 
Specifically, let $\{\pdist{\gamma}\}_{\gamma \in \bm{\Gamma}}$ be a set of parameterized distributions. Let $\pdist{\gamma^*}$ be a target distribution with target parameter $\gamma^*$. 
The \textit{diffusion score sensitivity index (DSSI)} of the target parameter $\gamma^*$ 
is defined as,
\begin{equation}\label{eq: DSSI}
    L(\gamma^*) = \inf_{\gamma \in \bm{\Gamma}}~ \frac{\ldsm(\pdist{\gamma},\pdist{\gamma^*})}{(\gamma-\gamma^*)^2}
\end{equation}
 
We note that both $\ldsm$ and $L(\gamma^*)$ depend on our choice of noise schedule.
If the target data distribution has mixture weight $\gamma^*$, and we learn a distribution with mixture parameter $\gamma$, then the squared error of mixture weight estimation $\gamma$ is bounded by the diffusion score matching loss divided by the DSSI. Thus, when the DSSI is small, our learned distribution may be highly misrepresentative of $\gamma^*$ even when the diffusion score matching loss is low. In contrast, when $L(\gamma^*)$ is large, we know that a small diffusion score matching loss indicates low error in the parameter $\gamma$. Our main contributions are as follows:
\begin{itemize}
    \item For general bimodal mixtures, we prove a strict minimax separation in mixture-weight recovery between estimators that use only the target-score data and estimators with access to score data at an intermediate, mode-overlapping noise level. (Theorem \ref{thm:sensitivity})
    \item For Gaussian mixtures, we prove a lower bound on the DSSI that holds in any dimension; thus, parameter estimation error is controlled by the DSM loss (Theorem \ref{thm: GMM Sensitivity Emerges}).
    
    \item We further empirically evaluate the DSSI on a real-world dataset. These experiments illustrate how design decisions relating to diffusion model training and inference-time sampling schedules can affect the DSSI and the fidelity of generated samples to the training data distribution.
\end{itemize}
\section{Preliminaries: score-based diffusion models}
\label{sec:prelim}
Score-based diffusion models sample from a target distribution on $\mathbb{R}^d$ by tracing out a one-parameter family of distributions $\{p_t\}_{t \in [0,1]}$ indexed by a diffusion time $t$, with $p_0$ being the target data distribution and $p_1 \approx \mathcal{N}(\mathbf{0}, \mathbf{I})$ the standard normal distribution. A \emph{forward process} progressively transforms samples from $p_0$ into samples from $p_1$ along this family. A \emph{reverse process} inverts the densities along the forward process, starting from $\mathcal{N}(0,\mathbf{I})$, using learned approximations of the scores of the intermediate marginals $p_t$.
Thus, score-based diffusions have 3 components: (a) the forward diffusion process and its corresponding reverse process, (b) a training objective for learning the marginal score functions $\nabla \log p_t$ of the intermediate distributions for $t\in [0,1]$,  and (c) a sampling procedure used to approximately generate samples from the target data distribution at inference time. In this section, we provide a brief description of each component.

\paragraph{Noising process.} 
Given a data distribution $p_0$ on $\mathbb{R}^d$, the noising process we consider is the {variance-preserving forward SDE} (VP-SDE)  proposed in \citet{song2021scorebasedgenerativemodelingstochastic}. Give a noise schedule $\set{\beta_t}$, this takes the form 
\begin{equation}
    \mathrm{d}{\vx}_t 
    \;=\; 
    -\tfrac{1}{2} \beta_t \, {\vx}_t\, \mathrm{d}t \;+\; \sqrt{\beta_t}\,\mathrm{d}\vect{w}_t,
    \qquad {\vx}_0 \sim p_0,\quad t \in [0,1],
    \label{eq:forward-sde}
\end{equation}
where $\vect{w}_t$ is a standard $d$-dimensional Wiener process. We denote by $p_t$ the marginal distribution of ${\vx}_t$. The forward conditional law of ${\vx}_t$ given ${\vx}_0$ under \eqref{eq:forward-sde} is Gaussian, 
\begin{equation}
    p_t({\vx}_t \mid {\vx}_0) 
    \;=\; 
    \mathcal{N}\!\left(
    {\vx}_t;\; \sqrt{\alpha_t}\,{\vx}_0,\; (1-\alpha_t)\,\vect{I}
    \right),
   \label{eq:forward-kernel}
\end{equation}
with conditional mean $\sqrt{\alpha_t} {\vx}_0$ and conditional covariance $(1-\alpha_t) \vect{I}$, where $\alpha_t=e^{-\int_0^t\beta_s \mathrm{d}s}$.\footnote{As a slight abuse of terminology, the sequences $\set{\alpha_t}$ and $\set{\beta_t}$ will interchangably be referred to as the ``noise schedule''.} 

\paragraph{Denoising process.}
\citet{ANDERSON1982313} established that the time-reversal of \eqref{eq:forward-sde} is itself an SDE:
\begin{equation}
    \mathrm{d}{\vx}_t 
    \;=\; 
    \left[-\tfrac{1}{2}{\beta_t}\,{\vx}_t \;-\; {\beta_t}\,\nabla_{{\vx}} \log p_t({\vx}_t)\right]
    \mathrm{d}t 
    \;+\; 
    \sqrt{{\beta_t}}\,\mathrm{d}\bar{\vect{w}}_t,
    \qquad {\vx}_1 \sim p_1,
    \label{eq:reverse-sde}
\end{equation}
where $\bar{\vect{w}}_t$ is the reverse-time Wiener process. 
Starting from ${\vx}_1 \sim p_1$ and integrating the reverse SDE with the exact score function from time $t=1$ to $t=0$ produces samples ${\vx}_0 \sim p_0$.  In practice, this reverse SDE is simulated starting from ${\vx}_1 \sim \mathcal{N}(0,\mathbf{I}) \approx p_1$ using the learned approximation $\vs_{\btheta}: \mathbb{R}^d \times [0,1] \rightarrow \mathbb{R}^d$ with parameters $\btheta  \in \boldsymbol{\Theta}$ trained such that $\vs_{\btheta}({\vx},t) \approx \nabla \log p_t({\vx})$ simultaneously for all noise levels to produce samples approximately distributed as the target $p_0$. 

For a given trained score model $\{\vs_{\btheta }(\cdot, t)\}_{t\in [0,1]}$ (or the true score function $\{\nabla \log p_t\}_{t\in [0,1]}$), generation of samples is an inference task: one may simulate the reverse SDE in \eqref{eq:reverse-sde} with any choice of consistent numerical integration schemes \citep{karras2022elucidating}.
The overall process has 
two distinct sources of error: the \emph{approximation error} incurred in learning the score
and the \emph{sampling error} incurred by the choice of sampling algorithm (e.g., discretization error, finite timesteps, etc.).

\section{Minimax bounds for parameter estimation from diffusion models}\label{sec: parameterized mixture families}
In this section, we show that when the score along the noising process exhibits sensitivity to $\gamma$, the information-theoretic barrier that disallows its recovery from the target score alone is removed. We prove that when the score is sensitive at some noise level, we obtain a lower parameter estimation error when compared to only using the target score. 
While stated explicitly for mixture weight as the parameter, the results in this section apply to any parameter. We consider only bimodal mixtures for simplicity, but make no distributional assumptions about the components. 
\paragraph{Mixture Family.}
Let $\pdist{0}$ and $\pdist{1}$ be two fixed probability distributions on $\mathbb{R}^d$. A probability density $\pdist{\gamma}, \gamma \in \bm{\Gamma} \subseteq (0,1),$ is a member of a family of two-component mixtures if it is of the form $\pdist{\gamma}\coloneqq (1-\gamma)\: \pdist{0} + \gamma \:\pdist{1}$.
The mixture parameter $\gamma$ controls the relative weights of the two component distributions, $\pdist{0}$ and $\pdist{1}$. We now consider the evolution of a mixture, $\pdist{\gamma},$ under a noising process or forward process in a score-based diffusion model. 
Specifically, when $p_0 \equiv \pdist{\gamma},$ we denote by $\pdist[t]{\gamma}$ the marginals of $\vx_t$ following the process \eqref{eq:forward-sde}. Let $\score{\gamma}$ denote the {score} of the density $\pdist{\gamma}$, and similar for marginal densities of $\vx_t$:
\begin{equation}\label{eq:score-and-noisyscore}
  \score{\gamma} := \nabla \log \pdist{\gamma} \qquad \text{and} \qquad \score[t]{\gamma} := \nabla \log \pdist[t]{\gamma}.  
\end{equation}
The relative weights of the mixture are preserved under the forward process: the noisy marginal $\pdist[t]{\gamma}$ is itself a $\gamma$-mixture of the noisy component marginals $\pdist[t]{0}$ and $\pdist[t]{1}$ (see \Cref{lemma:mixture-closure}). 
As a consequence the score of the noisy mixture density $\pdist[t]{\gamma}$ can be decomposed as follows,
\begin{align}
\score[t]{\gamma}({\vx})
    = w_t^{(\gamma,0)}({\vx})\,\score[t]{0}({\vx}) \;+\; w_t^{(\gamma,1)}({\vx})\,\score[t]{1}({\vx})
    \label{eq:mix score}
\end{align}
with component weights $w_t^{(\gamma,0)}({\vx}) = (1-\gamma)\,\pdist[t]{0}({\vx})/\pdist[t]{\gamma}({\vx})$ and $w_t^{(\gamma,1)}({\vx}) = \gamma\,\pdist[t]{1}({\vx})/\pdist[t]{\gamma}({\vx})$.
This decomposition enables a simplified theoretical analysis, as we see in the next section. 

\paragraph{Classical score matching loss/relative Fisher information.} The classical score matching (SM) loss measures the relative Fisher information between distributions \citep{10.5555/1046920.1088696}. For a pair of mixtures, 
\begin{align}
    \lsm(\pdist{\gamma'},\pdist{\gamma}) := \mathbb{E}_{\vx \sim \pdist{\gamma}} \|\score{\gamma'}(\vx)-\score{\gamma}(\vx) \|^2.
\end{align}
The relative Fisher information between a pair of mixture distributions with identical components but different mixture weights is small when the components are well-separated. In \cref{fig:scores}, the scores of Gaussian mixtures may appear nearly identical outside a low-probability region, an effect strengthened with increasing separation between mixture components as shown in \cref{fig: different loss functions}.   
Any sampling algorithm that relies solely on the score, such as Langevin sampling, exhibits near-identical behavior across distinct mixture distributions; consequently, the generated samples may not accurately reflect the target mixture weights.  

\paragraph{Diffusion score matching loss for mixtures.} 
The diffusion score matching (DSM) loss anneals the score matching loss across the noising process \eqref{eq:forward-sde} \citep{song2019generative}. For a pair of mixtures, 
\begin{align}
    \ldsm( \pdist{\gamma'},\pdist{\gamma}) &:=  \mathbb{E}_{\substack{t \sim \mathrm{Unif}(0,1)\\\vx\sim \pdist[t]{\gamma}}}[\lambda_t \|\score[t]{\gamma'}(\vx) - \score[t]{\gamma}(\vx)\|^2] 
    = \mathbb{E}_{t \sim \mathrm{Unif}(0,1)}[\lambda_t \: \lsm(\pdist[t]{\gamma'},\pdist[t]{\gamma})].  \label{eq:diffScoreMatchingLoss}
\end{align}
Here, $\lambda_t$ is a weighting factor set to $\lambda_t = 1-\alpha_t$ following \citet{song2021scorebasedgenerativemodelingstochastic}. The DSM loss measures the relative Fisher information over the entire sequence of noised densities generated by the process \eqref{eq:forward-sde}. Recall that the marginals $p_t$'s and scores $\vs_t$'s depend on the noise schedule $\{\alpha_t\}_t$, as shown in \eqref{eq:forward-kernel} and \eqref{eq:score-and-noisyscore}, and this way, the diffusion score matching loss $\ldsm$ is dependent on the choice of noise schedule. 
In practice, 
we may train a neural network $\vs_{\btheta}$ to approximate the score function. With some abuse of notation, we denote $\pdist{\btheta}$ the distribution\footnote{The superscript here indicates the trainable parameters $\btheta$ of the learned network.} of samples generated using the score functions $\{\vs_{\btheta}(\cdot, t)\}_{t\in (0,1)}$ for a fixed noise schedule as per \eqref{eq:reverse-sde}. Since the true score function $\score[t]{\gamma}$ and indeed the target distribution $\pdist{\gamma}$ are unknown,  calculating the diffusion score matching loss 
is intractable. 
Instead, we minimize the objective 
\begin{equation}\label{eqn: neural network loss}
\lnn(\pdist{\btheta}, \pdist{\gamma}):=\mathbb{E}_{\substack{t \sim \mathrm{Unif}(0,1)\\\vx_0\sim \pdist{\gamma}, \vx_t\sim \pdist[t]{\gamma}(\cdot | \vx_0)}}[\lambda_t \|\vect \vs_{\btheta}(\vx_t, t) - \nabla \log \pdist[t]{\gamma}(\vx_t|\vx_0)\|^2].
\end{equation}

As per \Cref{{lem:nn-dsm-ordering}}, there exists a constant $C\geq 0$, depending only on $\pdist{\gamma}$, such that $\lnn(\pdist{\btheta}, \pdist{\gamma})=\ldsm(\pdist{\btheta}, \pdist[t]{\gamma}) + C$. 
Thus, 
$\ldsm(\pdist{\btheta}, \pdist{\gamma}) \leq \lnn(\pdist{\btheta}, \pdist{\gamma})$.

\paragraph{Minimax theory for mixture weight recovery.} From classical minimax theory (see e.g., \citep{geer2000empirical} for a textbook exposition), we know that the error in any estimator of a parameter is lower-bounded by the informativeness of the data distribution about the parameter. Now, the data distribution loses information about the parameter in the noising process. How, then, would annealing benefit  arameter recovery, if the original mixture is well-separated and hence insensitive to parameter changes? The answer lies in the fact that diffusion models learn to approximate the \emph{scores} at various noise levels. By minimizing $\ldsm,$ diffusion models obtain information about the parameter through the scores even if such information is not present in the samples.  Our main result in this section shows the existence of a noise scale at which the score-data law is sensitive, provably enabling parameter recovery rates that are unattainable from target-score data alone.

\begin{definition}[Score oracle and score-restricted estimators] We consider a Gaussian White Noise (GWN) oracle from nonparametric regression \cite{tsybakov2008nonparametric} of the form: $\hat{\mathbf s}_{t}^{(\gamma)} = \score[t]{\gamma}\:  + \sigma_t \: W,$ where $\sigma_t > 0$ is a scalar and $W$ is an isonormal field indexed by $L^2(\pdist[t]{\gamma})$. The statistical experiment is to estimate $\gamma$ from the observation, $\hat{\mathbf s}^{(\gamma)}_t.$ Let $\qdist[t]{\gamma}$ denote the distribution of the observation, which is a Gaussian measure with mean $\score[t]{\gamma}.$ We define the set, $\Gamma_{\mathrm{score},t},$ of all measurable functions of the GWN observations to $(0,1)$ to be the class of ``score-restricted estimators''.  Further, we use $ \mathbb{E}_{\gamma^*,t}$ as a shorthand for expectations with respect to the score oracle model, $\qdist[t]{\gamma}.$ See also Remark \ref{remark:score-restricted-estimators} for more details on the estimator classes.
\label{def:scoreEstimator}
\end{definition}

We call $\Gamma_{\mathrm{score},t}$ \emph{score-restricted estimators} since they access the score vector field (corrupt by Gaussian noise) and not the target data samples. In other words, our class of estimators $\Gamma_{\mathrm{score},t}$ models parameter recovery achievable from having oracle access to a corrupt score vector field at time $t,$ reflecting the denoising/generation phase in diffusion models after scores have been trained. 
Since the observation model, $\qdist[t]{\gamma},$ is Gaussian, we can compute, fixing the indexing Hilbert space for $W$ to be $L^2(\pdist[t]{\gamma}),$ that $\mathrm{KL}(\qdist[t]{\gamma}||\qdist[t]{\gamma'}) = \lsm(\pdist[t]{\gamma'}, \pdist[t]{\gamma})/(2\sigma_t^2).$ See Remark \ref{remark:cameronMartin} for this computation. 

\begin{assumption}[Stable score-matching at time $t$] We assume that at some time $t > 0,$ the score observation law, $\qdist[t]{(\cdot)},$ satisfies $ (\gamma - \gamma')^2 \leq A_t^2 \: \mathrm{KL}(\qdist[t]{\gamma}||\qdist[t]{\gamma'})$ for some $A_t >0$ and  all $\gamma, \gamma' \in (0,1).$  
\label{assumption:continuity}
\end{assumption}

Since $\mathrm{KL}(\qdist[t]{\gamma}||\qdist[t]{\gamma'}) = \lsm(\pdist[t]{\gamma'}, \pdist[t]{\gamma})/(2\sigma_t^2),$ the above assumption implies that when the score matching loss is bounded above by some $C,$ then, the squared difference in the corresponding parameters, $(\gamma - \gamma')^2 \leq  \:A_t^2\: C / (2\sigma_t^2).$ 
Intuitively, this means that, for an 
$O(\sigma_t^2)$
constant $A_t,$ the score observations predict distances on the space of weight parameters.    Our next definition allows us to make this setting rigorous by comparing score observations at  time $t$ to a different behavior of the scores at time $0:$ the target score observations, $\hat{\mathbf s}^{(\gamma)},$ are insensitive in law to changes in $\gamma$ and therefore, do not predict distances on parameter space. In the next section, note that we obtain explicit estimates on $A_t$ for Gaussian mixtures. 

\begin{definition}[$(\epsilon,\delta)$-diffusion sensitivity]
\label{definition:diffusionSensitivity} We say that a mixture family $\{\pdist{\gamma}\}$ exhibits an $\epsilon$-diffusion sensitivity if for every $\gamma^* \in (0,1),$ there exist $\epsilon$-packing sets $S(\gamma^*) \subset (0,1)$ containing $N(\gamma^*,\epsilon) \geq 16$ elements with 
$\mathrm{KL}(\qdist{\gamma}||\qdist{\gamma'}) = \lsm(\pdist{\gamma'}, \pdist{\gamma})/(2\sigma_0^2) \leq \delta(\epsilon,N) \leq \epsilon^2/4,$ for all $\gamma, \gamma' \in S(\gamma^*)$ satisfying $|\gamma - \gamma'| \geq  4\sqrt{2}\: A_t \epsilon$. 
Note that $A_t$ here is as defined in Assumption \ref{assumption:continuity}.\end{definition}

\begin{definition}[$\epsilon$-diffusion sensitivity]
\label{definition:diffusionSensitivity} We say that a mixture family $\{\pdist{\gamma}\}$ exhibits $\epsilon$-diffusion sensitivity if for every $\gamma^* \in (0,1),$ there exist $\epsilon$-packing sets $S(\gamma^*) \subset (0,1)$ containing $N(\gamma^*,\epsilon) \geq 16$ elements with 
 $\mathrm{KL}(\qdist{\gamma}||\qdist{\gamma'}) = \lsm(\pdist{\gamma'}, \pdist{\gamma})/(2\sigma_0^2) \leq \epsilon^2/4,$ for all $\gamma, \gamma' \in S(\gamma^*)$ satisfying $|\gamma - \gamma'| \geq  4\sqrt{2}\: A_t \epsilon$. Note that $A_t$ here is as defined in Assumption \ref{assumption:continuity}.\end{definition}

\begin{theorem}
\label{thm:sensitivity}
Fix $t>0$ satifying Assumption \ref{assumption:continuity}. Let $\epsilon_{t} > 0$ be a covering radius that solves $\epsilon_{t}:= \sigma_t\: \sqrt{V(\epsilon_{t})},$ where $V(\epsilon_{t})$ is the $\epsilon_{t}$-covering entropy\footnote{Refer \Cref{def:cover-entropy} for an explicit statement} with respect to square root KL divergence of the class of Gaussian measures $\{\qdist[t]{\gamma}\}.$ 
     Suppose the bimodal mixture family satisfies \Cref{assumption:continuity} and exhibits an $\epsilon$-diffusion sensitivity with $\epsilon = \epsilon_{t}.$
    Then, we have the following minimax upper bound on the score-restricted estimators at time $t$, 
    \begin{align}
        \min_{\hat{\gamma} \in \Gamma_{\mathrm{score},t}}\max_{\gamma^*}~
        \mathbb{E}_{\gamma^*,t} |\gamma^* - \hat{\gamma}|^2 \leq 2\:A_t^2\: \epsilon_{t}^2\: .
        \label{eq:upperBound}
    \end{align} 
Simultaneously, if the packing set cardinality $N(\gamma^*, \epsilon_{t}) \geq e^{\sigma_t^2\:V(\epsilon_{t})},$ for all $\gamma^*,$ it  holds that the minimax lower bound achievable for target score-restricted estimators is the following:
    \begin{align}
\label{eq:lowerBound}
\min_{\hat{\gamma} \in \Gamma_{\mathrm{score},0}}\max_{\gamma^*}~ \mathbb{E}_{\gamma^*,0} |\gamma^* - \hat{\gamma}|^2 \geq 4\: A_t^2\: \epsilon_{t}^2.
    \end{align}
\end{theorem}


\paragraph{Proof.} The upper bound in \eqref{eq:upperBound} is a direct application of Theorem 2 from \citep{yangBarron}. Then, using a generalized Fano's method (Lemma 3 of \citep{binyu}) and an appropriate choice of $\delta$ as a function of $\epsilon$ in \cref{definition:diffusionSensitivity} that limits the packing entropy of the cover of $(0,1)$ with cardinality $e^{V(\epsilon_{t})},$ we obtain the desired lower bound in \eqref{eq:lowerBound}. We defer a more detailed proof to  Appendix \ref{sec:lowerUpperBounds}. \qed

Theorem \ref{thm:sensitivity} elucidates the conditions under which a diffusion model learns the mixture weights accurately when target score estimation followed by Langevin sampling does not. The above analysis 
applies to any bimodal mixture density and to any scalar parameter in place of the mixture weight. 
The main mechanism that enables parameter recovery is that the score field at some time $t$
changes more rapidly with parameter changes. In contrast, the target score at time 0 may not vary, creating an information-theoretic barrier that limits the accuracy of any
estimator. That is, no estimator can distinguish two close values of $\gamma$ using a finite number of target
score samples.

\begin{remark}
A key contribution of our work is to explain the surprising phenomenon in which a diffusion model can accurately recover a parameter even when the target score remains insensitive to small perturbations of the parameter. In response, we defined the notion of DSSI, \eqref{eq: DSSI}, which measures the sensitivity of the scores along the entire noising process, as opposed to only the target score, to infinitesimal changes in the parameter. The most interesting scenario that is directly relevant to explaining ``accurate parameter recovery despite target score insensitivity'' is one where an intermediate score is ``much more'' sensitive to the parameter than the score of the target distribution. Our definition of the $\epsilon$-sensitivity is designed to capture this scenario and further allows us to fully explain the information-theoretic mechanism underlying accurate recovery (Theorem \ref{thm:sensitivity}). 
\end{remark}




We remark that \citet{yangBarron} constructs a Bayes density estimator that belongs to the mixture class and achieves the upper bound \eqref{eq:upperBound}. Thus, we can conclude from Theorem \ref{thm:sensitivity} that using score observations at time $t$ leads to provably better parameter recovery.


\textbf{Resolving data-processing inequality.} Saying a family of distributions exhibits an $(\epsilon,\delta)$-sensitivity essentially means that more information about the parameter is created in the noising process, which, at first glance, appears to violate the data processing inequality. The law of noised samples, $\pdist[t]{\gamma},$ cannot be more informative about $\gamma$ at any time $t$ than $\pdist{\gamma}$ (the target mixture), due to the data processing inequality. More precisely, since $\pdist[t]{\gamma}$ results from a $\gamma$-independent convolution operation on $\pdist[0]{\gamma},$ the data processing inequality gives: 
$\mathrm{KL}(\pdist[t]{\gamma'}||\pdist[t]{\gamma}) \leq \mathrm{KL}(\pdist{\gamma'}||\pdist{\gamma}).$ But, our key idea here is that diffusion models use the score observation fields, $\hat{\mathbf s}^{(\gamma)}_t,$ to recover $\gamma$, even when $\hat{\mathbf s}^{(\gamma)}$ can not separate $\gamma$ values sufficiently. Specifically, under the assumptions of Theorem \ref{thm:sensitivity} on the mixture family, the score observations can be more informative about $\gamma$ at time $t$ than at time 0, i.e., $\mathrm{KL}(\qdist[t]{\gamma'}||\qdist[t]{\gamma}) \geq \mathrm{KL}(\qdist{\gamma'}||\qdist{\gamma}),$ without violating data processing. In other words, \cref{definition:diffusionSensitivity} defines a structural property of mixture distributions: the family, $\{\qdist[t]{\gamma}\}$ at intermediate $t$ can be more separated in $\gamma$ than at $t=0$ when mixture components do not overlap.
 In Lemma \ref{remark:dataProcessing}, we prove why this counterintuitive inequality does not violate data processing and why no estimator in $\Gamma_0$ can reconstruct the score observation model that $\Gamma_t$ has access to. 

\subsection{Effect of modified sensitivity to mixture weight}\label{sec: GMM Sens Exp}

While Theorem \ref{thm:sensitivity} provides a mechanism for parameter recovery, we now discuss experimental evidence for the same mechanism but using the practically computable diffusion score sensitivity index \eqref{eq: DSSI}, rather than the information-theoretic quantities used in \cref{definition:diffusionSensitivity}. Our numerical results show that the score sensitivity index quantifies how accurately a diffusion model captures the mixture weight of a Gaussian mixture model. We generate samples using the true analytical score function $\score[t]{\gamma^*}(\vect x)$ and a sequence of modified noise schedules with a coarse time discretization ($T=100$), designed to artificially reduce the diffusion score sensitivity index $L(\gamma^*)$ (for details, see \Cref{sec: Reparam}). Recall that $L(\gamma^*)$ and $\ldsm$ depend on the choice of noise schedule $\{\alpha_t\}_t$. 

We modify both linear and squared-cosine noise schedules and numerically estimate $L(\gamma^*)$ for these schedules.  \Cref{fig: gmm est gamma vs sensitivity index} shows the estimated mixture weight $\hat\gamma$ (calculated using Expectation-Maximization) plotted against $L(\gamma^*)$. We see that as the sensitivity index $L(\gamma^*)$ decreases, both the bias and variance in $\hat\gamma$ increase. This demonstrates that noise schedules with lower $L(\gamma^*)$ yield mixture-weight estimates that are less robust to discretization errors, even with a perfect score estimator.

On the other hand, the means and covariances of the modes are uniformly well-approximated across the different noise schedules. This shows that errors in estimating the mixture weight can occur without noticeable artifacts in the geometry of individual modes. Modern accelerated diffusion sampling schemes (e.g., \citep{song2021denoising}) leverage coarse time discretizations to reduce the computational cost of sampling while maintaining comparable sample fidelity. Our results suggest that such sampling approaches may be vulnerable to misrepresenting parameters of interest in the data distribution while appearing to generate accurate samples.

\begin{figure}
    \centering
    \begin{subfigure}[htbp]{0.24\linewidth}
        \includegraphics[width=\columnwidth]{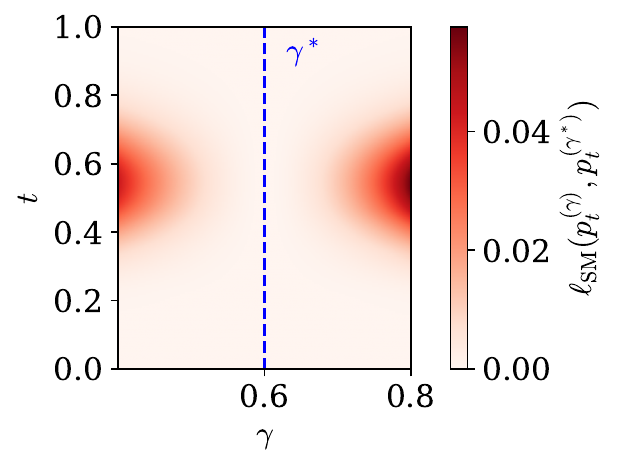}
        \caption{}
        \label{fig: gmm lin heatmap}
    \end{subfigure}
    \hfill
    \begin{subfigure}[htbp]{0.24\linewidth}
        \includegraphics[width=\columnwidth]{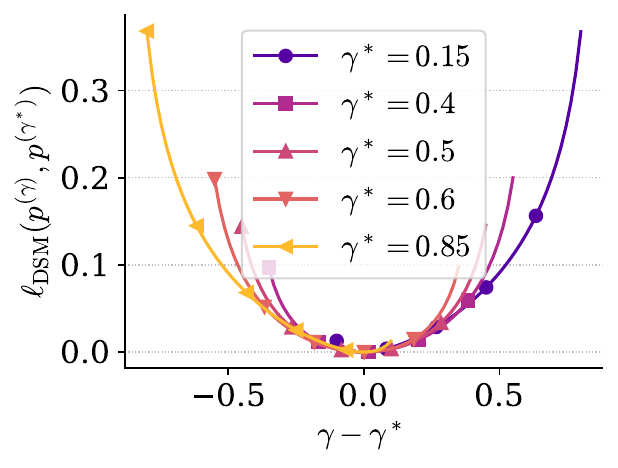}
        \caption{}
        \label{fig: gmm score sensitivity}
    \end{subfigure}
    \hfill    
    \begin{subfigure}[htbp]{0.24\linewidth}
        \includegraphics[width=\columnwidth]{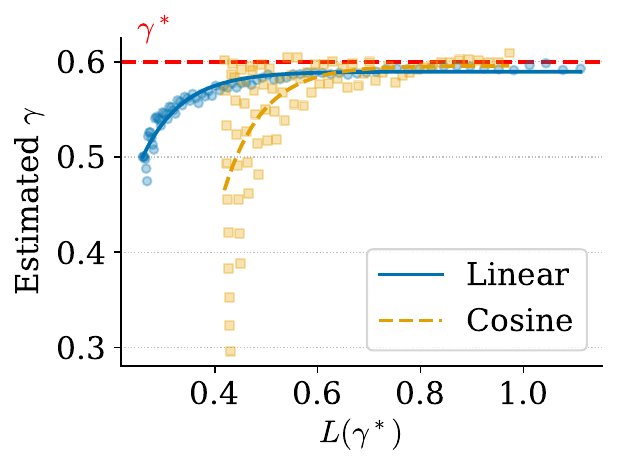}
        \caption{}
        \label{fig: gmm est gamma vs sensitivity index}
    \end{subfigure}
    \hfill    
    \begin{subfigure}[htbp]{0.24\linewidth}
        \includegraphics[width=\columnwidth]{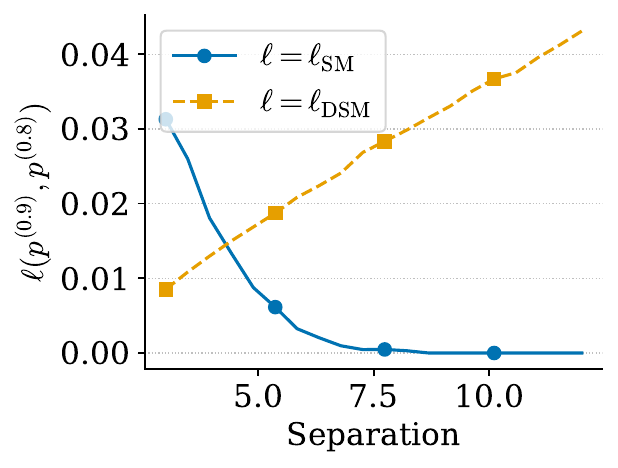}
        \caption{}
        \label{fig: different loss functions}
    \end{subfigure}
    
    \caption{
    Score sensitivity for Gaussian mixture models in dimension $d=10$. (a) $\lsm(\pdist[t]{\gamma}, \pdist[t]{\gamma^*})$ as a function of $t$ and $\gamma$. 
    (b) $\ldsm(\pdist{\gamma}, \pdist{\gamma^*})$ vs. $\gamma-\gamma^*$ for multiple $\gamma^*$, showing non-negligable sensitivity throughout. (c) Mixture weight estimates $\hat\gamma$ from generated data for $\gamma^*=0.6$ under modified linear and cosine noise schedules (see \Cref{sec: Reparam}). The bias and variance of $\hat\gamma$ increase as $L(\gamma^*)$ decreases. (d) $\lsm$ and $\ldsm$ as a function of separation for $\gamma^*=0.8$ and $\gamma=0.9$. $\ldsm$ remains nonnegligible while $\lsm$ rapidly decays.
    }
    \label{fig:gmm prac sens}
\end{figure}

\section{Diffusion sensitivity to GMM mixture weights}

We show that for a bimodal isotropic Gaussian mixture in any dimension, the diffusion score sensitivity index (DSSI) $L(\gamma^*)$ 
is non-negligible. As a result, a diffusion model with a well-trained score generates samples with the correct weights even for extreme values of $\gamma$.  

\begin{theorem}[Diffusion sensitivity in bimodal Gaussian mixtures]\label{thm: GMM Sensitivity Emerges}
Consider the family $\pdist{\gamma}(\vx) =(1-\gamma) \; \CalN(\vx; \mu_0, \vect I)  + \gamma \;\CalN(\vx; \mu_1, \vect I)$ and a linear noise schedule $\beta_t=\beta_1 t$ for $\beta_1>0$. 
Let $\alpha_t=\exp\fitparenth{\int_0^t\beta_s \mathrm{d}s}$. 
Then for any $a \in (0,1/2)$
\begin{align}
   \min_{\hat \gamma \in [a,1-a]}\frac{\ldsm(\pdist{\hat\gamma},\pdist{\gamma^*})}{(\hat\gamma-\gamma^*)^2}>\frac{(1-\Phi(\norm{\mu_1-\mu_0}\sqrt{\alpha_1}))\fitparenth{\sqrt{k^2+16}-k}}{7\beta_1\gamma^*(1-\gamma^*)}\exp\fitparenth{-(k^2+k)/4} 
   \label{eq:main_thm}
\end{align}
with $k:=\left|\log \frac{1-a}{a}\right|+\left|\log\frac{1-\gamma^*}{\gamma^*}\right|$ and $\Phi(\cdot)$ the CDF of $\mathcal{N}(0,1)$. 



\end{theorem}

\emph{Proof Sketch:}
The proof proceeds in three steps. We first decompose $\ldsm(\pdist{\hat\gamma}, \pdist{\gamma^*})$ by mixture component \eqref{eq: mode decomposition} to reduce the problem to lower-bounding the expectations over paths from $\pdist{0}$ and $\pdist{1}$ independently. Second, using the Probability Flow ODE (which is equivalent in probability to the forward process) \citep{song2021scorebasedgenerativemodelingstochastic}, we find an analytical form for $\norm{\score[t]{\hat\gamma}-\score[t]{\gamma^*}}^2$ along paths $\vect{z}_t$ starting from $\pdist{0}$ (\Cref{prop:exact norm expr}). Finally, we identify the interval of $t$ values where this expression is bounded below (\Cref{prop: lower bound over t}), integrate over this region of $t$ (\Cref{prop: Lower bound integral}), and finally integrate over initial conditions $\vect {y} _ {0} $. The bound is dimension-independent because for a given initial condition $\vect {y}_{0} $ starting from e.g. mode $\pdist{0}$, the bound only depends on $\inprod{\vect y_0-\mu_0, \mu_1-\mu_0}/\norm{\mu_1-\mu_0}\sim \mathcal{N}(0,1)$ in any dimension. For full details, see \Cref{Proof of Theorem 1}.
\qed

\begin{remark}
    Recall that if $\ldsm(\pdist{\gamma},\pdist{\gamma^*})<\delta$ then $(\gamma-\gamma^*)^2 \le \delta/L(\gamma^*)$ by the definition of $L(\gamma^*)$ in \eqref{eq: DSSI}.
    \Cref{thm: GMM Sensitivity Emerges} shows that $L(\gamma^*)$ cannot be too small under mild conditions for two-component Gaussian mixture models. 
    Further, recall that the neural network loss $\lnn$ \eqref{eqn: neural network loss} is an upper bound for $\ldsm$, and so if we train a neural network to an error $\lnn(\pdist{\btheta},\pdist{\gamma})<\delta$ and our neural network learns a score corresponding to the mixture distributions $\{\pdist[t]{\gamma'}\}_{t\in (0,1)}$, then $\ldsm(\pdist{\gamma'},\pdist{\gamma^*})<\delta$. 
    Our theory more directly implies that any score model which closely approximates the true score $\score{\gamma^*}$ in $L^2$ will have small $\ldsm$, and thus its sample distribution will accurately reflect $\gamma^*$.
\end{remark}

\begin{remark}
    If, when generating new samples at inference time, we alter the noise schedule $\{\alpha_t\}$, this may alter the mixture proportions in the generated samples. Specifically, altering the noise schedule changes $\ldsm$ and $L(\gamma^*)$, as noted below \eqref{eq:diffScoreMatchingLoss}.
    Hence, similar loss under $\ldsm$ for two different noise schedules could potentially correspond to very different guarantees for the accuracy of the recovered mixture weight $\hat\gamma$. Discretizing the reverse process SDE accumulates time-discretization errors regardless of score accuracy, meaning noise schedules that lower $L(\gamma^*)$ may worsen mixture-weight errors in the generative distribution, even for well-trained diffusion models.
    \Cref{sec: GMM Sens Exp} showcases this effect for GMMs. In \Cref{sec:sensitivity-real-data}, we conduct further experiments on real-world data. 
\end{remark}

\paragraph{Effect of parameters on DSSI.} The DSSI $L(\gamma^*)$ depends on the mode separation $\norm{\mu_1-\mu_0}$, true mixture weight $\gamma^*$, and choice of noise schedule $\set{\alpha_t}_t$. First, consistent with the fact that $\ldsm(\pdist{\gamma}, \pdist{\gamma^*})$ is an increasing function of the separation (\Cref{fig: different loss functions}), \Cref{fig: score sensitivity over separation} shows $L(\gamma^*)$ is an increasing function of the mode separation. Second, although \Cref{thm: GMM Sensitivity Emerges} is only stated for linear noise schedules, different continuous noise schedules are equivalent up to time-reparameterization, guaranteeing $\lsm(\pdist[t]{\hat\gamma}, \pdist[t]{\gamma^*})$ is nonnegligible for some range of $t$ under any noise schedule (\Cref{fig: gmm cos heatmap}). Note however that $L(\gamma^*)$ is defined relative to $\ldsm$, which averages $\lsm(\pdist[t]{\hat\gamma},\pdist[t]{\gamma^*})$ over $t\in[0,1]$, meaning a noise schedule which reduces this range of $t$ to a small interval $[a,a+\epsilon]$ scales down $L(\gamma^*)$ roughly proportional to $\epsilon$, regardless of the peak value of $\lsm(\pdist[t]{\hat\gamma},\pdist[t]{\gamma^*})$. 
Third, \Cref{fig: score sensitivity ratio over gamma^* and gamma hat} shows that for mode separation $\norm{\mu_1-\mu_0}=9$, $L(\gamma^*)$ is bounded below by 0.37 across all $\gamma^*$, implying that the squared prediction error in $\gamma$ is at most $\approx 2.7$ times larger than the squared diffusion score error. Notably, this bound is dimension-independent, and hence this holds for bimodal isotropic Gaussian mixtures in any dimension. In \Cref{sec: Other Parameters}, we showcase additional experiments on parameters of Gaussian mixtures other than the mixture weight.

\begin{figure}
\centering
    \begin{subfigure}[htbp]{0.3\linewidth}
    \includegraphics[width=\columnwidth]{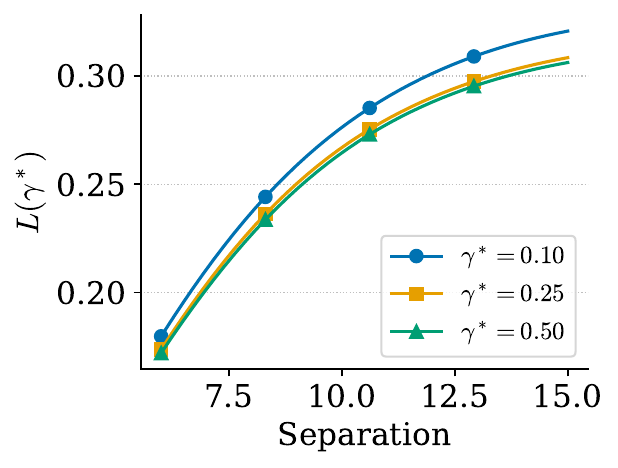}
    \caption{}
    \label{fig: score sensitivity over separation}
    \end{subfigure}
    \hfill    \begin{subfigure}[htbp]{0.3\linewidth}
        \includegraphics[width=\columnwidth]{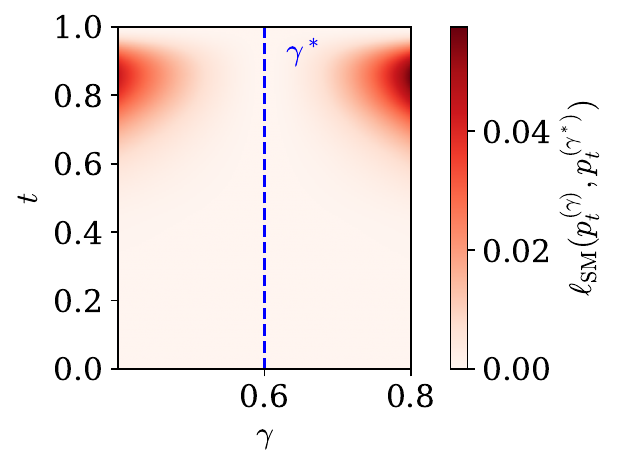}
        \caption{}
        \label{fig: gmm cos heatmap}
    \end{subfigure}
    \hfill    \begin{subfigure}[htbp]{0.3\linewidth}
    \includegraphics[width=\columnwidth]{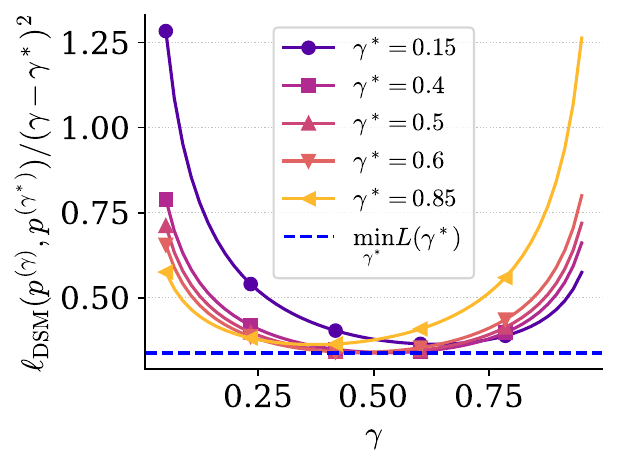}
        \caption{}
        \label{fig: score sensitivity ratio over gamma^* and gamma hat}
    \end{subfigure}
\caption{
Effect of Gaussian mixture model parameters on score sensitivity in dimension $d=10$.  (a) $L(\gamma^*)$ vs. mode separation for multiple $\gamma^*\leq 0.5$ (by symmetry $L(\gamma^*)=L(1-\gamma^*)$). $L(\gamma^*)$ is an increasing function of separation.
(b) $\lsm(\pdist[t]{\gamma}, \pdist[t]{\gamma^*})$ as a function of $t$ and $\gamma$ for a cosine noise schedule. Compared to a linear noise schedule, the noise schedule only affects the range of $t$ in the forward process at which the score is sensitive to $\gamma$. 
(c) $\ldsm(\pdist{\gamma}, \pdist{\gamma^*})/(\gamma -\gamma^*)^2$ vs. $\gamma$; $L(\gamma^*)\geq 0.27$ for all $\gamma^*$, with the minimum at $\gamma^*=0.5$ . 
}\label{fig:sensitivity_index_gmm}
\end{figure}

\subsection{Related Work}

\paragraph{Multimodal capabilities of diffusion models.}
The ability of diffusion models to sample multimodal targets varies substantially across common choices of neural architecture \citep{hakemi2025deeper}, training recipes \citep{karras2022elucidating}, and sampling algorithms \citep{roos2026met, zhang2025collapse}. Existing evaluation metrics \citep{sajjadi2018assessing, kynkaanniemi2019improved, djolonga2020precision, naeem2020reliable, alaa2022faithful} are typically structure-agnostic and may fail to reliably detect mode coverage \citep{raisa2025position, stein2023exposing}. 
In contrast, the diffusion score sensitivity index \eqref{eq: DSSI} provides a rigorous framework to reason about multimodal capability.

\paragraph{Learning GMMs with Diffusion Score Matching.} 
Diffusion score matching (DSM) is \textit{in principle} expressive enough to learn Gaussian mixture: gradient descent recovers component means in the balanced isotropic regime\footnote{The mixture model $q\coloneqq \sum_{i=1}^k \frac{1}{k}\, \mathcal{N}(\boldsymbol{\mu}^{(i)},\, \mathbf{I})$ where each component has equal mixture weight and identity covariance} \citep{shah2023learning}, 
learning algorithms can be tailored for density estimation of GMMs with polynomial sample complexity \citep{gatmiry25b, chen2024learning}. Here, \citet{gatmiry25b} develop noise-sensitivity bounds similar in spirit to ours. Unlike the aforementioned articles, our work focuses on the ability of diffusion models to sample from target GMMs rather than on density estimation. 
When the learned score is $L^2$-accurate, the resulting diffusion sampler is provably close to the target in total variation \citep{chen2023sampling, lee2023convergence}. However, this may be insufficient to recover structural parameters that shape the geometry of distributions in low-probability regions. 
\citet{li2024softmixturedenoisingexpressive} suggest diffusion models may fail to sample multimodal targets due to the lack of expressivity of the Gaussian transition steps in the reverse diffusion process. 
Our work complements \citet{li2024softmixturedenoisingexpressive} by identifying a different mechanism for multimodal sampling, the diffusion score sensitivity index \eqref{eq: DSSI}, and by providing additional insights into the role of noise schedules in alleviating mode amplification.

\paragraph{Benefits of Annealing.} 
Classical (unannealed) score matching \citep{10.5555/1046920.1088696} is unsuitable for generative modeling of multimodal targets: the asymptotic efficiency of the score matching estimator relative to the maximum-likelihood estimator is governed by the log-Sobolev (LS) constant \citep{koehler2022statisticalefficiencyscorematching}, which grows exponentially in mode separation for Gaussian mixtures \citep{chen2021dimensionfreelogsobolevinequalitiesmixture}. \citet{song2019generative} proposed annealing over noise scale to handle multimodality, inspired by \citet{neal_annealed_2001}. 
\citet{biroli2024dynamical} and \citet{raya2023symmetry} observe a phase transition at intermediate noise scales linked to the emergence of multimodal structure in practice. 
Amidst the empirical evidence, \citet{qin2024fit} establishes the first formal statistical benefit of annealing in the context of diffusion. \citet{qin2024fit} adapts techniques from the acceleration of Markov Chain mixing to derive a better score matching loss, with the resulting outcome closely resembling the diffusion score matching loss. While \citet{qin2024fit} establishes sample complexity bounds, our work addresses a complementary question of parameter recovery from generated samples. 

\section{Sensitivity in real datasets}\label{sec:sensitivity-real-data}

We empirically show that mixture weight information is captured in the score functions of intermediate forward process distributions, even outside the Gaussian mixture setting. We additionally explore how altering the noise schedule affects the DSSI and empirically demonstrate that a small DSSI value quantitatively degrades mixture-weight recovery on non-Gaussian data. This is particularly salient to methods like \citep{songddim2021, lu2022dpmsolver, karras2022elucidating} where noise schedules are selected to increase the speed of inference-time sampling; we show that such methods have the potential to amplify some modes over others, augmenting the recent observation by \citet{zhang2025collapse}. 

To study this, we train a simple autoencoder on the ones and eights from MNIST (\citep{lecun2010mnist}). Our target distribution $\pdist{\gamma^*}$ is then the mixture distribution of latent representations of eights and ones ($\pdist{1}$ is the distribution of ones). We set $\gamma^*=0.4$ and use deep ResNets \citep{He2015DeepRL} with a linear noise schedule as our score matching models.

\begin{figure}[ht]
    \centering
    \begin{subfigure}[h]{0.3\columnwidth}
        \centering
        \includegraphics[width=\columnwidth]{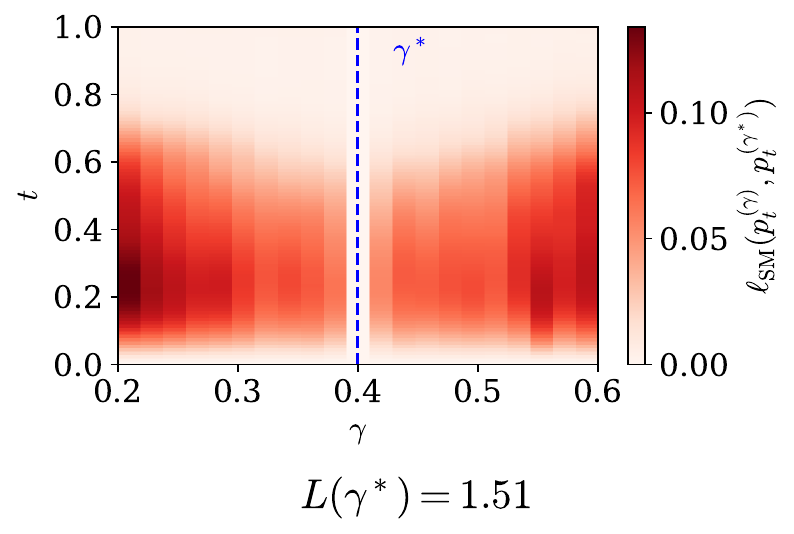}
    \end{subfigure}%
    \begin{subfigure}[h]{0.3\columnwidth}
        \centering
        \includegraphics[width=\columnwidth]{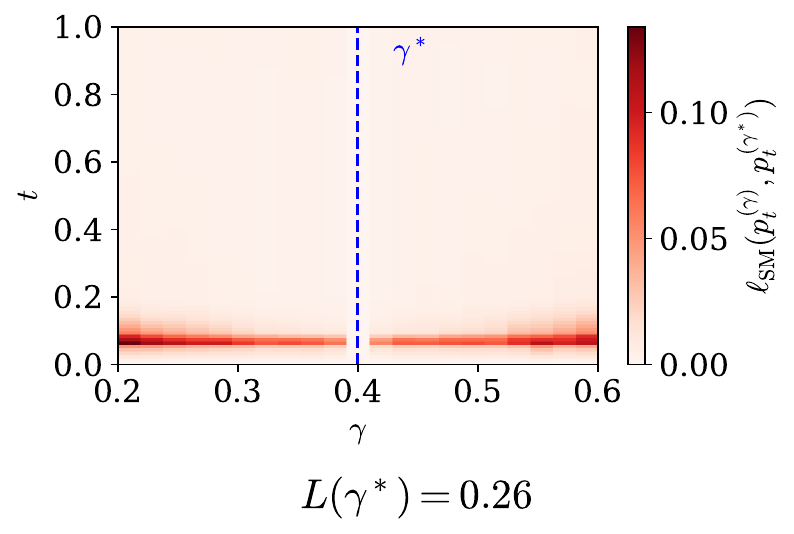}
    \end{subfigure}%
    \begin{subfigure}[h]{0.3\columnwidth}
        \centering
        \includegraphics[width=\columnwidth]{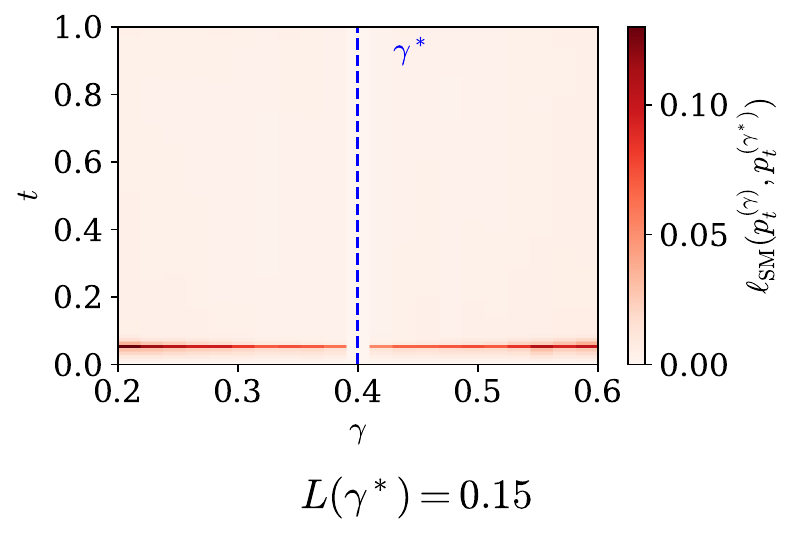}
    \end{subfigure}%
    \\
    \begin{subfigure}[h]{0.3\columnwidth}
        \centering
        \includegraphics[width=\columnwidth]{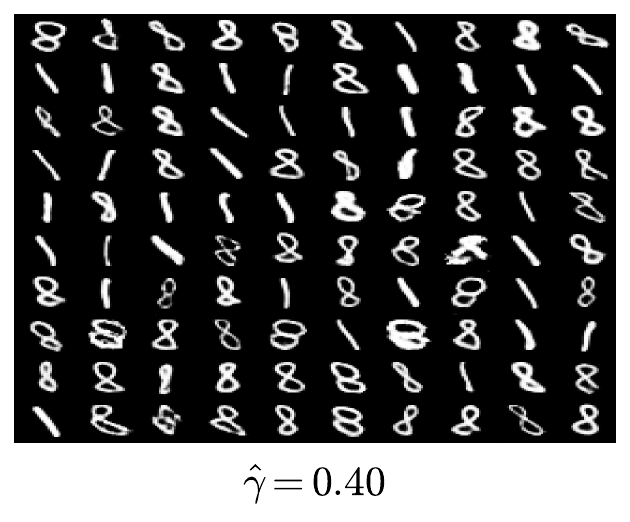}
        \caption{}
        \label{fig:mnist40}
    \end{subfigure}%
    \begin{subfigure}[h]{0.3\columnwidth}
        \centering
        \includegraphics[width=\columnwidth]{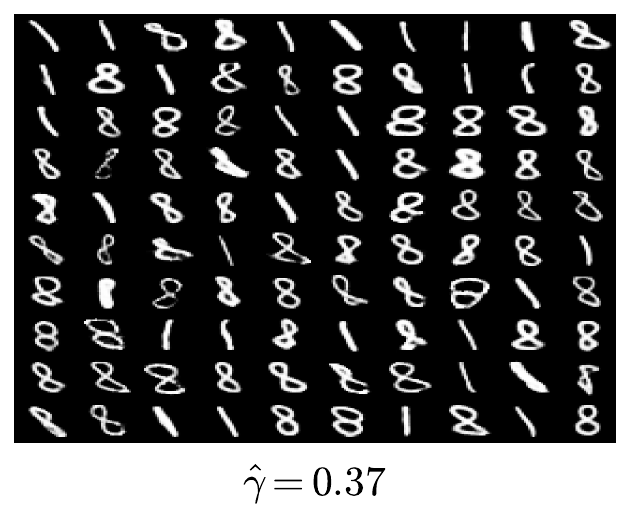}
        \caption{}
        \label{fig:mnist37}
    \end{subfigure}%
    \begin{subfigure}[h]{0.3\columnwidth}
        \centering
        \includegraphics[width=\columnwidth]{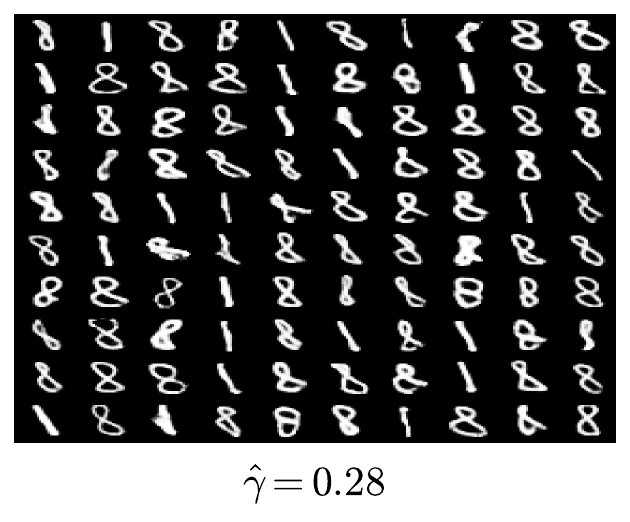}
        \caption{}
        \label{fig:mnist28}
    \end{subfigure}%
    \caption{
    Influence of noise schedule on mixture weight recovery in MNIST. (a) Linear noise schedule; (b-c) two synthetic noise schedules designed to reduce $L(\gamma^*)$ (see \Cref{sec: Reparam}). Top row: $\lsm(\pdist[t]{\gamma},\pdist[t]{\gamma^*})$ as a function of $t$ and $\gamma$ estimated using trained diffusion models, with $\gamma^*=0.4$. Bottom row: generated samples using each schedule, with estimated $\hat\gamma$ below. Larger $L(\gamma^*)$ yields more accurate mixture weight recovery.
    }
    \label{fig: MNIST experiment}
\end{figure}

\paragraph{Sensitivity at intermediate $t$.} In \Cref{fig: MNIST experiment}, the top row displays the score sensitivity over time and $\gamma$ using trained diffusion models (each with a different noise schedule) to approximate the true score function (see \Cref{sec: MNIST Exp Details} for more details). The first column uses the same linear noise schedule as in \Cref{fig: gmm lin heatmap}. Note that relative to the Gaussian mixture example, $\lsm(\pdist[t]{\gamma}, \pdist[t]{\gamma^*})$ is nonnegligable over a wider range of $t$ values.  
We numerically estimate that the sensitivity index is $\approx 5\times$ larger than in the Gaussian mixture model example, so our theory guarantees a five-fold smaller error in the predicted mixture weight for the same level of error under $\ldsm$ as in the Gaussian setting.

\paragraph{Sensitivity affects parameter recovery.} 
In \Cref{fig:mnist40}, the bottom row shows generated samples from a diffusion model with $T=100$. We estimate the value of $\gamma^*$ using generated samples and a trained classifier, shown below the generated samples. The generated samples accurately reflect the true value of the mixture weight, with only a $1\%$ relative error in estimating $\gamma^*$. 

To demonstrate the effect of the score sensitivity index $L(\gamma^*)$ on the error in estimating $\gamma^*$, we perform an ablation study on the trained diffusion model by modifying the noise schedule to reduce $L(\gamma^*)$ as in \Cref{sec: GMM Sens Exp} (for details see \Cref{sec: Reparam}). \Cref{fig:mnist37} and \Cref{fig:mnist28} show the score sensitivity and generated samples for the modified noise schedules, along with the corresponding estimates $\hat\gamma$ from the samples. When the score sensitivity index reduces by a factor of $10$ (\Cref{fig:mnist28}), we see that the estimate of $\gamma^*$ has a relative error of $31\%$. 
Surprisingly, the samples from the modified noise schedules have comparable fidelity to those from the linear noise schedule (\Cref{fig:mnist40}). This is in line with the theoretical and empirical observations that say that diffusion models under errors and accelerated sampling schemes can still accurately learn the geometry of the support of the target distribution \citep{songddim2021, nichol2021improveddenoisingdiffusionprobabilistic, stanczuk2022your,chandramoorthy2026when}. More details on our MNIST experiments can be found in \Cref{sec: MNIST Exp Details}. 

\section{Conclusion}

We investigate the ability of 
diffusion models to recover the relative mode amplitudes of a multimodal target, despite the insensitivity of the score of the target distribution to mixture weights under high separation. 
Our analysis clarifies the impact of annealing: the score of the noisy target distribution across noise scales provides the necessary information for accurate sampling of multimodal distributions. Our experiments indicate that inference-time sampling algorithms can be tuned for diversity and fidelity, suggesting a mechanism for tailoring
sampling algorithms to data distributions arising in scientific applications.

\paragraph{Limitations.} We prove a quantitative relationship between the sensitivity index and the diffusion score error only in the Gaussian mixture setting. Theorem \ref{thm:sensitivity} provides an information theoretic mechanism that cannot be computationally verified since it involves minimax bounds. Further, evaluation of the diffusion score-sensitivity index requires knowledge of the true parameter of interest $\gamma^*$ and the score function of the target distribution $\score[t]{\gamma^*}$. Thus, the application to benchmark datasets requires surrogate generative models in place of the unknown target score.

\bibliographystyle{unsrtnat} 
\bibliography{refs}

@incollection{tsybakov2008nonparametric,
  title={Nonparametric estimators},
  author={Tsybakov, Alexandre B},
  booktitle={Introduction to Nonparametric Estimation},
  pages={1--76},
  year={2008},
  publisher={Springer}
}

@book{geer2000empirical,
  title={Empirical Processes in M-estimation},
  author={Geer, Sara A},
  volume={6},
  year={2000},
  publisher={Cambridge university press}
}

@article{yangBarron,
author={Yuhong Yang and Andrew Barron},
title={Information-theoretic determination of minimax rates of convergence},
journal={The Annals of Statistics},
volume={27},
number={5}, pages={1564--1599}, 
year={1999}
}

@article{lanusse2021deep,
  title={Deep generative models for galaxy image simulations},
  author={Lanusse, Fran{\c{c}}ois and Mandelbaum, Rachel and Ravanbakhsh, Siamak and Li, Chun-Liang and Freeman, Peter and P{\'o}czos, Barnab{\'a}s},
  journal={Monthly Notices of the Royal Astronomical Society},
  volume={504},
  number={4},
  pages={5543--5555},
  year={2021},
  publisher={Oxford University Press}
}

@article{karras2022elucidating,
  title={Elucidating the design space of diffusion-based generative models},
  author={Karras, Tero and Aittala, Miika and Aila, Timo and Laine, Samuli},
  journal={Advances in neural information processing systems},
  volume={35},
  pages={26565--26577},
  year={2022}
}

@article{lecun2010mnist,
  title={MNIST handwritten digit database},
  author={LeCun, Yann and Cortes, Corinna and Burges, CJ},
  journal={ATT Labs [Online]. Available: http://yann.lecun.com/exdb/mnist},
  volume={2},
  year={2010}
}

@article{sajjadi2018assessing,
  title={Assessing generative models via precision and recall},
  author={Sajjadi, Mehdi SM and Bachem, Olivier and Lucic, Mario and Bousquet, Olivier and Gelly, Sylvain},
  journal={Advances in neural information processing systems},
  volume={31},
  year={2018}
}

@inproceedings{roos2026met,
  title={How I Met Your Bias: Investigating Bias Amplification in Diffusion Models},
  author={Roos, Nathan and Iakovleva, Ekaterina and Gjergji, Ani and Pastore, Vito Paolo and Tartaglione, Enzo},
  booktitle={Proceedings of the IEEE/CVF Winter Conference on Applications of Computer Vision},
  pages={5374--5383},
  year={2026}
}

@article{hakemi2025deeper,
  title={Deeper Diffusion Models Amplify Bias},
  author={Hakemi, Shahin and Akhtar, Naveed and Hassan, Ghulam Mubashar and Mian, Ajmal},
  journal={arXiv preprint arXiv:2505.17560},
  year={2025}
}

@inproceedings{
song2021denoising,
title={Denoising Diffusion Implicit Models},
author={Jiaming Song and Chenlin Meng and Stefano Ermon},
booktitle={International Conference on Learning Representations},
year={2021},
url={https://openreview.net/forum?id=St1giarCHLP}
}

@article{ludwig2024machine,
  title={Machine learning as a tool for hypothesis generation},
  author={Ludwig, Jens and Mullainathan, Sendhil},
  journal={The Quarterly Journal of Economics},
  volume={139},
  number={2},
  pages={751--827},
  year={2024},
  publisher={Oxford University Press}
}

@article{posner2025judge,
author = {Eric A. Posner and Shivam Saran},
title ={Judge AI: A Case-Study of Large Language Models as Judges},

journal = {Journal of Law \& Empirical Analysis},
volume = {0},
number = {0},
pages = {2755323X261433614},
year = {2026},
doi = {10.1177/2755323X261433614},

URL = { 
    
        https://doi.org/10.1177/2755323X261433614
    
    

},
eprint = { 
    
        https://doi.org/10.1177/2755323X261433614
    
    

}
}

@article{teo2025generative,
  title={Generative artificial intelligence in medicine},
  author={Teo, Zhen Ling and Thirunavukarasu, Arun James and Elangovan, Kabilan and Cheng, Haoran and Moova, Prasanth and Soetikno, Brian and Nielsen, Christopher and Pollreisz, Andreas and Ting, Darren Shu Jeng and Morris, Robert JT and others},
  journal={Nature medicine},
  volume={31},
  number={10},
  pages={3270--3282},
  year={2025},
  publisher={Nature Publishing Group US New York}
}

@article{fahrner2025generative,
  title={The generative era of medical AI},
  author={Fahrner, L John and Chen, Emma and Topol, Eric and Rajpurkar, Pranav},
  journal={Cell},
  volume={188},
  number={14},
  pages={3648--3660},
  year={2025},
  publisher={Elsevier}
}

@article{loni2025review,
  title={A review on generative AI models for synthetic medical text, time series, and longitudinal data},
  author={Loni, Mohammad and Poursalim, Fatemeh and Asadi, Mehdi and Gharehbaghi, Arash},
  journal={npj Digital Medicine},
  volume={8},
  number={1},
  pages={281},
  year={2025},
  publisher={Nature Publishing Group UK London}
}

@misc{han2025learning,
      title={Learning noisy tissue dynamics across time scales}, 
      author={Ming Han and John Devany and Michel Fruchart and Margaret L. Gardel and Vincenzo Vitelli},
      year={2025},
      eprint={2510.19090},
      archivePrefix={arXiv},
      primaryClass={cond-mat.soft},
      url={https://arxiv.org/abs/2510.19090}, 
}

@article{price2023gencast,
   title={Probabilistic weather forecasting with machine learning}, DOI={https://doi.org/10.1038/s41586-024-08252-9}, journal={Nature}, author={Price, Ilan and Sanchez-Gonzalez, Alvaro and Alet, Ferran and Andersson, Tom R. and El-Kadi, Andrew and Masters, Dominic and Ewalds, Timo and Stott, Jacklynn and Mohamed, Shakir and Battaglia, Peter and Lam, Remi and Willson, Matthew}, year={2024}, month={Dec} }

@inproceedings{qin2024fit,
  title={Fit like you sample: sample-efficient generalized score matching from fast mixing diffusions},
  author={Qin, Yilong and Risteski, Andrej},
  booktitle={The Thirty Seventh Annual Conference on Learning Theory},
  pages={4413--4457},
  year={2024},
  organization={PMLR}
}

@misc{song2021scorebasedgenerativemodelingstochastic,
  author = {Song, Yang and Sohl-Dickstein, Jascha and Kingma, Diederik P and Kumar, Abhishek and Ermon, Stefano and Poole, Ben},
  booktitle = {International Conference on Learning Representations},
  title = {Score-Based Generative Modeling through Stochastic Differential Equations},
  url = {https://openreview.net/forum?id=PxTIG12RRHS},
  year = 2021
}

@article{ANDERSON1982313,
title = {Reverse-time diffusion equation models},
journal = {Stochastic Processes and their Applications},
volume = {12},
number = {3},
pages = {313-326},
year = {1982},
issn = {0304-4149},
doi = {https://doi.org/10.1016/0304-4149(82)90051-5},
url = {https://www.sciencedirect.com/science/article/pii/0304414982900515},
author = {Brian D.O. Anderson}
}

@misc{chen2021dimensionfreelogsobolevinequalitiesmixture,
    title = {Dimension-free log-Sobolev inequalities for mixture distributions},
    journal = {Journal of Functional Analysis},
    volume = {281},
    number = {11},
    pages = {109236},
    year = {2021},
    issn = {0022-1236},
    doi = {https://doi.org/10.1016/j.jfa.2021.109236},
    url = {https://www.sciencedirect.com/science/article/pii/S0022123621003189},
    author = {Hong-Bin Chen and Sinho Chewi and Jonathan Niles-Weed}
}

@inproceedings{
koehler2022statisticalefficiencyscorematching,
title={Statistical Efficiency of Score Matching: The View from Isoperimetry},
author={Frederic Koehler and Alexander Heckett and Andrej Risteski},
booktitle={The Eleventh International Conference on Learning Representations },
year={2023},
url={https://openreview.net/forum?id=TD7AnQjNzR6}
}

@inproceedings{
li2024softmixturedenoisingexpressive,
title={Soft Mixture Denoising: Beyond the Expressive Bottleneck of Diffusion Models},
author={Yangming Li and Boris van Breugel and Mihaela van der Schaar},
booktitle={The Twelfth International Conference on Learning Representations},
year={2024},
url={https://openreview.net/forum?id=aaBnFAyW9O}
}

@InProceedings{nichol2021improveddenoisingdiffusionprobabilistic,
  title = 	 {Improved Denoising Diffusion Probabilistic Models},
  author =       {Nichol, Alexander Quinn and Dhariwal, Prafulla},
  booktitle = 	 {Proceedings of the 38th International Conference on Machine Learning},
  pages = 	 {8162--8171},
  year = 	 {2021},
  editor = 	 {Meila, Marina and Zhang, Tong},
  volume = 	 {139},
  series = 	 {Proceedings of Machine Learning Research},
  month = 	 {18--24 Jul},
  publisher =    {PMLR},
  url = 	 {https://proceedings.mlr.press/v139/nichol21a.html}
}

@Inbook{binyu,
author="Yu, Bin",
title="Assouad, Fano, and Le Cam",
bookTitle="Festschrift for Lucien Le Cam: Research Papers in Probability and Statistics",
year="1997",
publisher="Springer New York",
address="New York, NY",
pages="423--435",
isbn="978-1-4612-1880-7",
doi="10.1007/978-1-4612-1880-7_29",
url="https://doi.org/10.1007/978-1-4612-1880-7_29"
}

@inproceedings{NIPS2017_8a1d6947,
 author = {Heusel, Martin and Ramsauer, Hubert and Unterthiner, Thomas and Nessler, Bernhard and Hochreiter, Sepp},
 booktitle = {Advances in Neural Information Processing Systems},
 editor = {I. Guyon and U. Von Luxburg and S. Bengio and H. Wallach and R. Fergus and S. Vishwanathan and R. Garnett},
 pages = {},
 publisher = {Curran Associates, Inc.},
 title = {GANs Trained by a Two Time-Scale Update Rule Converge to a Local Nash Equilibrium},
 url = {https://proceedings.neurips.cc/paper_files/paper/2017/file/8a1d694707eb0fefe65871369074926d-Paper.pdf},
 volume = {30},
 year = {2017}
}

@inproceedings{
shah2023learning,
title={Learning Mixtures of Gaussians Using the {DDPM} Objective},
author={Kulin Shah and Sitan Chen and Adam Klivans},
booktitle={Thirty-seventh Conference on Neural Information Processing Systems},
year={2023},
url={https://openreview.net/forum?id=aig7sgdRfI}
}

@inproceedings{chen2024learning,
  title = 	 {Learning general Gaussian mixtures with efficient score matching},
  author =       {Chen, Sitan and Kontonis, Vasilis and Shah, Kulin},
  booktitle = 	 {Proceedings of Thirty Eighth Conference on Learning Theory},
  pages = 	 {1029--1090},
  year = 	 {2025},
  editor = 	 {Haghtalab, Nika and Moitra, Ankur},
  volume = 	 {291},
  series = 	 {Proceedings of Machine Learning Research},
  month = 	 {30 Jun--04 Jul},
  publisher =    {PMLR},
  url = 	 {https://proceedings.mlr.press/v291/chen25e.html}
}

@inproceedings{gatmiry25b,
  title = 	 {Learning Mixtures of Gaussians Using Diffusion Models},
  author =       {Gatmiry, Khashayar and Kelner, Jonathan and Lee, Holden},
  booktitle = 	 {Proceedings of Thirty Eighth Conference on Learning Theory},
  pages = 	 {2403--2456},
  year = 	 {2025},
  editor = 	 {Haghtalab, Nika and Moitra, Ankur},
  volume = 	 {291},
  series = 	 {Proceedings of Machine Learning Research},
  month = 	 {30 Jun--04 Jul},
  publisher =    {PMLR},
  url = 	 {https://proceedings.mlr.press/v291/gatmiry25b.html}
}

@inproceedings{
chen2023sampling,
title={Sampling is as easy as learning the score: theory for diffusion models with minimal data assumptions},
author={Sitan Chen and Sinho Chewi and Jerry Li and Yuanzhi Li and Adil Salim and Anru Zhang},
booktitle={The Eleventh International Conference on Learning Representations },
year={2023},
url={https://openreview.net/forum?id=zyLVMgsZ0U_}
}

@inproceedings{lee2023convergence,
  title     = {Convergence of Score-Based Generative Modeling for General Data Distributions},
  author    = {Lee, Holden and Lu, Jianfeng and Tan, Yixin},
  booktitle = {Proceedings of the 34th International Conference on Algorithmic Learning Theory},
  series    = {Proceedings of Machine Learning Research},
  volume    = {201},
  pages     = {946--985},
  year      = {2023},
  publisher = {PMLR},
  note      = {arXiv:2209.12381}
}

@article{biroli2024dynamical,
  title   = {Dynamical Regimes of Diffusion Models},
  author  = {Biroli, Giulio and Bonnaire, Tony and De Bortoli, Valentin and M{\'e}zard, Marc},
  journal = {Nature Communications},
  volume  = {15},
  pages   = {9957},
  year    = {2024},
  doi     = {10.1038/s41467-024-54281-3}
}

@inproceedings{raya2023symmetry,
  title     = {Spontaneous Symmetry Breaking in Generative Diffusion Models},
  author    = {Raya, Gabriel and Ambrogioni, Luca},
  booktitle = {Advances in Neural Information Processing Systems},
  volume    = {36},
  year      = {2023},
  url       = {https://arxiv.org/abs/2305.19693}
}

@inproceedings{songddim2021,
  title     = {Denoising Diffusion Implicit Models},
  author    = {Song, Jiaming and Meng, Chenlin and Ermon, Stefano},
  booktitle = {International Conference on Learning Representations},
  year      = {2021},
  note      = {arXiv:2010.02502}
}

@inproceedings{lu2022dpmsolver,
  title     = {{DPM}-Solver: A Fast {ODE} Solver for Diffusion Probabilistic Model Sampling in Around 10 Steps},
  author    = {Lu, Cheng and Zhou, Yuhao and Bao, Fan and Chen, Jianfei and Li, Chongxuan and Zhu, Jun},
  booktitle = {Advances in Neural Information Processing Systems},
  volume    = {35},
  year      = {2022},
  note      = {arXiv:2206.00927}
}

@article{yakowitz1968identifiability,
  title   = {On the Identifiability of Finite Mixtures},
  author  = {Yakowitz, Sidney J. and Spragins, John D.},
  journal = {The Annals of Mathematical Statistics},
  volume  = {39},
  number  = {1},
  pages   = {209--214},
  year    = {1968},
  doi     = {10.1214/aoms/1177698520}
}

@article{neal_annealed_2001,
	title = {Annealed importance sampling},
	volume = {11},
	issn = {1573-1375},
	url = {https://doi.org/10.1023/A:1008923215028},
	doi = {10.1023/A:1008923215028},
	number = {2},
	journal = {Statistics and Computing},
	author = {Neal, Radford M.},
	month = apr,
	year = {2001},
	pages = {125--139},
}

@article{bouchet2019rare,
  title   = {Rare Event Sampling Methods},
  author  = {Bouchet, Freddy and Rolland, Joran and Wouters, Jeroen},
  journal = {Chaos: An Interdisciplinary Journal of Nonlinear Science},
  volume  = {29},
  number  = {8},
  pages   = {080402},
  year    = {2019},
  doi     = {10.1063/1.5120509}
}

@article{addison2024cpmgem,
  title   = {Machine Learning Emulation of Precipitation from km-Scale {UK} Regional Climate Simulations Using a Diffusion Model},
  author  = {Addison, Henry and Kendon, Elizabeth and Ravuri, Suman and Aitchison, Laurence and Watson, Peter A. G.},
  journal = {Journal of Advances in Modeling Earth Systems},
  year    = {2026},
  doi     = {10.1029/2025MS005140},
  note    = {arXiv:2407.14158}
}

@article{He2015DeepRL,
  title={Deep Residual Learning for Image Recognition},
  author={Kaiming He and X. Zhang and Shaoqing Ren and Jian Sun},
  journal={2016 IEEE Conference on Computer Vision and Pattern Recognition (CVPR)},
  year={2015},
  pages={770-778},
  url={https://api.semanticscholar.org/CorpusID:206594692}
}

@inproceedings{mudur2023cosmological,
  title     = {Cosmological Field Emulation and Parameter Inference with Diffusion Models},
  author    = {Mudur, Nayantara and Cuesta-Lazaro, Carolina and Finkbeiner, Douglas P.},
  booktitle = {NeurIPS Workshop on Machine Learning and the Physical Sciences},
  year      = {2023},
  note      = {arXiv:2312.07534}
}

@misc{zhang2025collapse,
  title         = {On the Collapse Errors Induced by the Deterministic Sampler for Diffusion Models},
  author        = {Zhang, Yi and Liao, Zhenyu and Wu, Jingfeng and Zou, Difan},
  year          = {2025},
  eprint        = {2508.16154},
  archivePrefix = {arXiv},
  primaryClass  = {cs.LG},
  url           = {https://arxiv.org/abs/2508.16154}
}

@book{10.5555/1162264,
author = {Bishop, Christopher M.},
title = {Pattern Recognition and Machine Learning (Information Science and Statistics)},
year = {2006},
isbn = {0387310738},
publisher = {Springer-Verlag},
address = {Berlin, Heidelberg}
}

@inproceedings{song2019generative,
title = {Generative Modeling by Estimating Gradients of the
Data Distribution},
author = {Song, Yang and Ermon, Stefano},
booktitle = {Advances in Neural Information Processing Systems},
volume = {32},
pages = {11895--11907},
year = {2019}
}

@article{6795935,
  author={Vincent, Pascal},
  journal={Neural Computation}, 
  title={A Connection Between Score Matching and Denoising Autoencoders}, 
  year={2011},
  volume={23},
  number={7},
  pages={1661-1674},
  doi={10.1162/NECO_a_00142}}

@inproceedings{
chandramoorthy2026when,
title={When and how can inexact generative models still sample from the data manifold?},
author={Nisha Chandramoorthy and Adriaan A. de Clercq},
booktitle={The Thirty-ninth Annual Conference on Neural Information Processing Systems},
year={2026},
url={https://openreview.net/forum?id=QBjuYL4gAX}
}

@article{10.5555/1046920.1088696,
author = {Hyv\"{a}rinen, Aapo},
title = {Estimation of Non-Normalized Statistical Models by Score Matching},
year = {2005},
issue_date = {12/1/2005},
publisher = {JMLR.org},
volume = {6},
issn = {1532-4435},
journal = {J. Mach. Learn. Res.},
month = dec,
pages = {695–709},
numpages = {15}
}

@inproceedings{NEURIPS2018_c6ede20e,
 author = {Lee, Holden and Risteski, Andrej and Ge, Rong},
 booktitle = {Advances in Neural Information Processing Systems},
 editor = {S. Bengio and H. Wallach and H. Larochelle and K. Grauman and N. Cesa-Bianchi and R. Garnett},
 pages = {},
 publisher = {Curran Associates, Inc.},
 title = {Beyond Log-concavity: Provable Guarantees for Sampling Multi-modal Distributions using Simulated Tempering Langevin Monte Carlo},
 url = {https://proceedings.neurips.cc/paper_files/paper/2018/file/c6ede20e6f597abf4b3f6bb30cee16c7-Paper.pdf},
 volume = {31},
 year = {2018}
}

@article{stanczuk2022your,
  title={Your diffusion model secretly knows the dimension of the data manifold},
  author={Stanczuk, Jan and Batzolis, Georgios and Deveney, Teo and Sch{\"o}nlieb, Carola-Bibiane},
  journal={arXiv preprint arXiv:2212.12611},
  year={2022}
}

@inproceedings{kynkaanniemi2019improved,
  title     = {Improved Precision and Recall Metric for Assessing 
               Generative Models},
  author    = {Kynk{\"a}{\"a}nniemi, Tuomas and Karras, Tero and Laine, Samuli 
               and Lehtinen, Jaakko and Aila, Timo},
  booktitle = {Advances in Neural Information Processing Systems},
  volume    = {32},
  year      = {2019}
}

@inproceedings{djolonga2020precision,
  title     = {Precision-Recall Curves Using Information Divergence Frontiers},
  author    = {Djolonga, Josip and Lucic, Mario and Cuturi, Marco 
               and Bachem, Olivier and Bousquet, Olivier and Gelly, Sylvain},
  booktitle = {Proceedings of the 23rd International Conference on 
               Artificial Intelligence and Statistics},
  series    = {Proceedings of Machine Learning Research},
  volume    = {108},
  pages     = {2550--2559},
  year      = {2020}
}

@inproceedings{naeem2020reliable,
  title     = {Reliable Fidelity and Diversity Metrics for Generative Models},
  author    = {Naeem, Muhammad Ferjad and Oh, Seong Joon and Uh, Youngjung 
               and Choi, Yunjey and Yoo, Jaejun},
  booktitle = {Proceedings of the 37th International Conference on 
               Machine Learning},
  series    = {Proceedings of Machine Learning Research},
  volume    = {119},
  pages     = {7176--7185},
  year      = {2020}
}

@inproceedings{alaa2022faithful,
  title     = {How Faithful is your Synthetic Data? Sample-level Metrics for 
               Evaluating and Auditing Generative Models},
  author    = {Alaa, Ahmed M. and van Breugel, Boris and Saveliev, Evgeny S. 
               and van der Schaar, Mihaela},
  booktitle = {Proceedings of the 39th International Conference on 
               Machine Learning},
  series    = {Proceedings of Machine Learning Research},
  pages     = {290--306},
  year      = {2022},
  note      = {arXiv:2102.08921}
}

@inproceedings{
raisa2025position,
title={Position: All Current Generative Fidelity and Diversity Metrics are Flawed},
author={Ossi R{\"a}is{\"a} and Boris van Breugel and Mihaela van der Schaar},
booktitle={Forty-second International Conference on Machine Learning Position Paper Track},
year={2025},
url={https://openreview.net/forum?id=DMRrbb36r5}
}

@inproceedings{
stein2023exposing,
title={Exposing flaws of generative model evaluation metrics and their unfair treatment of diffusion models},
author={George Stein and Jesse C. Cresswell and Rasa Hosseinzadeh and Yi Sui and Brendan Leigh Ross and Valentin Villecroze and Zhaoyan Liu and Anthony L. Caterini and Eric Taylor and Gabriel Loaiza-Ganem},
booktitle={Thirty-seventh Conference on Neural Information Processing Systems},
year={2023},
url={https://openreview.net/forum?id=08zf7kTOoh}
}

@article{latuszynski2025mcmc,
  title={MCMC for multi-modal distributions},
  author={{\L}atuszy{\'n}ski, Krzysztof and Moores, Matthew T and Stumpf-F{\'e}tizon, Timoth{\'e}e},
  journal={arXiv preprint arXiv:2501.05908},
  year={2025}
}

@ARTICLE{1100034,
  author={Alspach, D. and Sorenson, H.},
  journal={IEEE Transactions on Automatic Control}, 
  title={Nonlinear Bayesian estimation using Gaussian sum approximations}, 
  year={1972},
  volume={17},
  number={4},
  pages={439-448},
  doi={10.1109/TAC.1972.1100034}}

@article{zhuang2020adabelief,
  title={AdaBelief Optimizer: Adapting Stepsizes by the Belief in Observed Gradients},
  author={Zhuang, Juntang and Tang, Tommy and Ding, Yifan and Tatikonda, Sekhar C and Dvornek, Nicha and Papademetris, Xenophon and Duncan, James},
  journal={Advances in Neural Information Processing Systems},
  volume={33},
  year={2020}
}

@article{kidger2021equinox,
    author={Patrick Kidger and Cristian Garcia},
    title={{E}quinox: neural networks in {JAX} via callable {P}y{T}rees and filtered transformations},
    year={2021},
    journal={Differentiable Programming workshop at Neural Information Processing Systems 2021}
}

@misc{jax2018github,
  author = {James Bradbury and Roy Frostig and Peter Hawkins and Matthew James Johnson and Yash Katariya and Chris Leary and Dougal Maclaurin and George Necula and Adam Paszke and Jake Vander{P}las and Skye Wanderman-{M}ilne and Qiao Zhang},
  title = {{JAX}: composable transformations of {P}ython+{N}um{P}y programs},
  url = {http://github.com/jax-ml/jax},
  version = {0.3.13},
  year = {2018},
}

@Article{Hunter2007matplotlib,
  Author    = {Hunter, J. D.},
  Title     = {Matplotlib: A 2D graphics environment},
  Journal   = {Computing in Science \& Engineering},
  Volume    = {9},
  Number    = {3},
  Pages     = {90--95},
  publisher = {IEEE COMPUTER SOC},
  doi       = {10.1109/MCSE.2007.55},
  year      = 2007
}

@misc{wandb,
title = {Experiment Tracking with Weights and Biases},
year = {2020},
note = {Software available from wandb.com},
url={https://www.wandb.com/},
author = {Biewald, Lukas},
}

@ARTICLE{2020SciPy-NMeth,
  author  = {Virtanen, Pauli and Gommers, Ralf and Oliphant, Travis E. and
            Haberland, Matt and Reddy, Tyler and Cournapeau, David and
            Burovski, Evgeni and Peterson, Pearu and Weckesser, Warren and
            Bright, Jonathan and {van der Walt}, St{\'e}fan J. and
            Brett, Matthew and Wilson, Joshua and Millman, K. Jarrod and
            Mayorov, Nikolay and Nelson, Andrew R. J. and Jones, Eric and
            Kern, Robert and Larson, Eric and Carey, C J and
            Polat, {\.I}lhan and Feng, Yu and Moore, Eric W. and
            {VanderPlas}, Jake and Laxalde, Denis and Perktold, Josef and
            Cimrman, Robert and Henriksen, Ian and Quintero, E. A. and
            Harris, Charles R. and Archibald, Anne M. and
            Ribeiro, Ant{\^o}nio H. and Pedregosa, Fabian and
            {van Mulbregt}, Paul and {SciPy 1.0 Contributors}},
  title   = {{{SciPy} 1.0: Fundamental Algorithms for Scientific
            Computing in Python}},
  journal = {Nature Methods},
  year    = {2020},
  volume  = {17},
  pages   = {261--272},
  adsurl  = {https://rdcu.be/b08Wh},
  doi     = {10.1038/s41592-019-0686-2},
}

@misc{hang2024improvednoiseschedulediffusion,
      title={Improved Noise Schedule for Diffusion Training}, 
      author={Tiankai Hang and Shuyang Gu and Xin Geng and Baining Guo},
      year={2024},
      eprint={2407.03297},
      archivePrefix={arXiv},
      primaryClass={cs.CV},
      url={https://arxiv.org/abs/2407.03297}, 
}

\newpage 
\begin{appendices}
\crefalias{section}{appendix}
\crefalias{subsection}{appendix}

\section{Notation}

{\renewcommand{\arraystretch}{1.5}
\begin{tabular}{ |p{15mm}| p{77mm}| p{35mm} |}
\hline
 Symbol & Definition & In words
 \\
 \hline
 $\mathcal{P}(p_0, p_1)$ 
 & 
 $\set{\pdist{\gamma}: \pdist{\gamma} = \gamma \; p_1+(1-\gamma) \; p_0,\ \gamma \in (0,1)}$ 
 & 
 Family of bimodal mixtures of $p_0$ and $p_1$
 \\
 \hline
 $p_t$ 
 & 
 $\mathrm{d} \vx_t = -\frac{\beta_t}{2} \vx_t \mathrm{d}t + \sqrt{\beta_t}\, \mathrm{d}W_t \text{ and } \vx_0 \sim p \Longrightarrow \vx_t \sim p_t$
 & 
 Distribution of samples at time $t$ in the forward process starting from $p$
 \\
 \hline
 $\lnn(\vs_\btheta, p)$ 
 & 
 $\expect[{ \vx_0 \sim p,\ t,\ \vx_t \mid \vx_0}]{\lambda_t\, \norm{\vs_\btheta(\vx_t, t) - \nabla \log \pdist[t]{\gamma^{*}}(\vx_t \mid \vx_0)}^2}$ 
 & 
 Diffusion model training objective
 \\
 \hline
 $\lsm(p,q)$ 
 & 
 $\expect[{\vx \sim q}]{\norm{\nabla \log p(\vx) - \nabla \log q(\vx)}^2}$ 
 & 
 Classical score matching loss
 \\
 \hline
 $\ldsm(p,q)$ 
 & 
 $\begin{aligned}[t]
   &\expect[{t \in (0,1),\ \vx_t \sim q_t}]{\lambda_t\, \norm{\nabla \log p_t(\vx_t) - \nabla \log q_t(\vx_t)}^2} \\
   &=\expect[{t \in (0,1)}]{\lambda_t\lsm(p,q)}
 \end{aligned}$ 
 & 
 Diffusion score matching loss
 \\
 \hline
 $L(\gamma^*)$ 
 & 
 $\displaystyle \min_{\gamma'} \frac{\ldsm(\pdist{\gamma'},\pdist{\gamma^*})}{(\gamma^* - \gamma')^2}$
 & 
 Diffusion score sensitivity index
 \\
 \hline
 $\vs_\btheta(\vx, t)$ 
 & 
 $\forall\, \btheta \in \bm{\Theta} = \RR^{\#\,\text{Parameters}},\ \vs_\btheta : \RR^d \times [0,1] \rightarrow \RR^d$ 
 & 
 Neural network with parameters $\btheta$
 \\
 \hline
\end{tabular}
}

\section{Comparison of sampling algorithms for multimodal distributions}

Sampling from multimodal distributions using classical Langevin sampling is known to be challenging \citep{NEURIPS2018_c6ede20e, latuszynski2025mcmc}. For a two-component Gaussian mixture, \Cref{fig: different loss functions} shows that as mode separation increases, the sensitivity of the classical score matching loss to the mixture weight decays rapidly to 0 while the diffusion score matching loss does not. As a consequence, sampling algorithms based solely on the target score are unsuitable for sampling from multimodal distributions. \Cref{fig:Langevin_vs_Diffusion} illustrates this effect, comparing classical Langevin sampling to diffusion sampling. 


\begin{figure}[ht]
    \centering
    \begin{subfigure}[htbp]{0.19\linewidth}
        \includegraphics[width=\columnwidth]{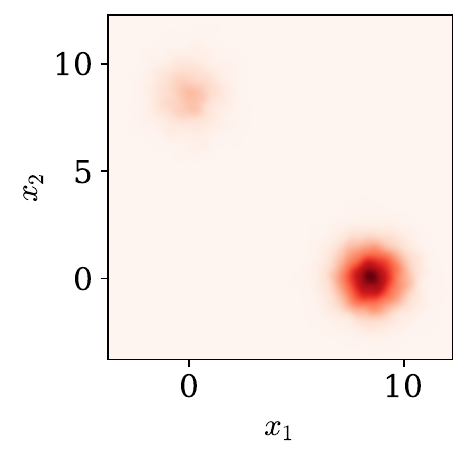}
        \caption{}
        \label{fig: true distribution kde}
    \end{subfigure}
    \begin{subfigure}[htbp]{0.19\linewidth}
        \includegraphics[width=\columnwidth]{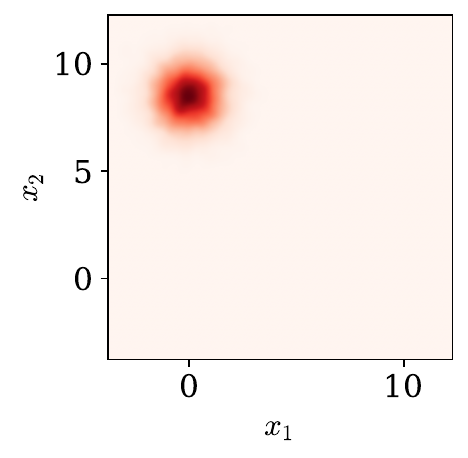}
        \caption{}
        \label{fig: langevin sampling kde}
    \end{subfigure}
    \begin{subfigure}[htbp]{0.19\linewidth}
        \includegraphics[width=\columnwidth]{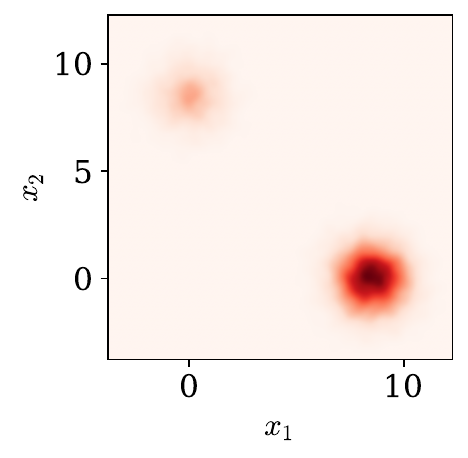}
        \caption{}
        \label{fig: analytic diffusion kde}
    \end{subfigure}
    \begin{subfigure}[htbp]{0.19\linewidth}
        \includegraphics[width=\columnwidth]{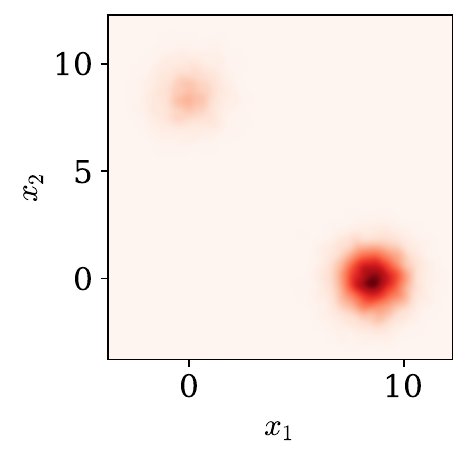}
        \caption{}
        \label{fig: learned diffusion kde}
    \end{subfigure}
    \caption{Comparison between different sampling approaches on a high-separation Gaussian mixture model. (a) shows a KDE of the true data distribution. (b)-(c) show, in order, KDEs of samples using (b) Langevin sampling (run to time $T=10^5$), (c) diffusion sampling using the analytic score function, and (d) diffusion sampling using a trained neural network. We can see that, unlike Langevin sampling, diffusion sampling accurately reflects the relative masses of the two modes. }
    \label{fig:Langevin_vs_Diffusion}
\end{figure}

\section{Minimax upper and lower bounds}
\label{sec:lowerUpperBounds}
Here we provide a more detailed proof of Theorem \ref{thm:sensitivity} and provide additional remarks on the Theorem and the underlying assumptions. We begin with a remark on score-restricted estimators (Definition \ref{def:scoreEstimator}) and explain the motivation for their definition.

In minimax theory, and in particular, Fano's method, the accuracy of an estimator for a parameter learned from a finite number of data samples is inherently limited by how well the data samples can distinguish between two parameter values. This distinguishability is often measured in terms of the KL divergence between the data distributions at two nearby parameter values. At first glance, the data samples given to a diffusion model appear to be the target samples, and as noted previously, the data distributions with well-separated modes do not distinguish between nearby values of weight parameters. Further, note
$\mathrm{KL}(\pdist[t]{\gamma}||\pdist[t]{\gamma'}) \leq
\mathrm{KL}(\pdist{\gamma}||\pdist{\gamma'}).$ As a result, it might seem that the generated samples from the denoising process do not accurately recover the parameter (particularly, given the role that the KL divergence plays in Fano's inequality \citep{binyu}), but our empirical results disagree with this expectation. 
In our empirical results, we see that the noising/denoising process affects parameter recovery compared with score approximation followed by Langevin dynamics. A key idea that helps resolve this mismatch is to recognize that the data a diffusion model uses to accurately estimate parameters is the noisy score (approximated by a neural network), not the noisy sample data. In other words, the generated samples, from the reverse dynamics, are a function of the noisy scores (interpolated by neural networks). These neural network interpolations at some time $t$ can have a distribution that distinguishes between parameter values more than at time $0$. In that sense, the time $t$ scores contain sufficient information to generate samples corresponding to a nearby parameter value, while the time $0$ score (target score) may not. 

Our score-restricted estimators define the notion of recovering $\gamma^*$ from ``score data'' rather than sample data. Such a notion models parameter estimation from the generated samples in score-based generative models where the score field is first approximated and then used for sampling. To distinguish Langevin sampling from diffusion models, we define two separate classes of estimators: one class has access to only a target score oracle and the other class has access to a score oracle at an intermediate time along the noising process. Hence, the estimator classes $\Gamma_{\mathrm{score},t}$ and $\Gamma_{\mathrm{score},0}$ differ in the data distributions given to them, which are respectively, $(\qdist[t]{\gamma})$ and $(\qdist[0]{\gamma}).$ In Lemma \ref{remark:dataProcessing}, we show that the oracle model $\qdist[t]{\gamma}$ cannot be reconstructed in a $\gamma$-independent manner from the oracle manner $\qdist[0]{\gamma},$ and this induces a separation between the two estimator classes.

\begin{remark}    \label{remark:score-restricted-estimators}
 While in Theorem \ref{thm:sensitivity}, we are only comparing $\Gamma_{\mathrm{score},0}$ and a set of estimators that use score data at one other time $t,$ the result continues to hold if instead under a specific redefinition of the two estimator classes. Let $t$ still be a time at which the diffusion sensitivity definition holds for the values of $\epsilon$ and $\delta$ defined in Theorem \ref{thm:sensitivity}. The estimator class $\Gamma_{\mathrm{score},0}$ contains any sampling scheme that does not use samples from $\qdist[t']{\gamma}$, for any time except $t' = 0.$ On the other hand, the estimators in $\Gamma_{\mathrm{score},0}$ must include sampling (denoising) steps that access score data from $\qdist[t]{\gamma}.$

\end{remark}

\begin{remark}
    \label{remark:cameronMartin} To compute the KL divergence, $\mathrm{KL}(\qdist[t]{\gamma}||\qdist[t]{\gamma'})$, we first recognize that both Gaussian measures are on $L^2(\pdist[t]{\gamma}).$ Hence, by Cameron-Martin formula, the Radon-Nikodym derivative, $d\qdist[t]{\gamma}/d\qdist[t]{\gamma'}.$ We can take expectations with respect to $\qdist[t]{\gamma}$ to arrive at the formula $\mathrm{KL}(\qdist[t]{\gamma}||\qdist[t]{\gamma'}) = \|\score[t]{\gamma} - \score[t]{\gamma'}\|^2_{L^2(\pdist[t]{\gamma}}/(2\sigma_t^2).$ This resembles the KL divergence between finite dimensional Gaussian distributions with the same covariance. 
\end{remark}

\begin{lemma}
\label{remark:dataProcessing} There is no $\gamma$-independent post processing of score observations at time 0, $\hat{\mathbf s}^{(\gamma)}_0,$ that can reproduce  $\hat{\mathbf s}^{(\gamma)}_t.$
\end{lemma}
\begin{proof}
We can prove why by contradiction. If there exists a $\gamma$-independent processing of $\qdist{\gamma}$ that produces the distribution $\qdist[t]{\gamma},$ the data processing inequality gives $\mathrm{KL}(\qdist[t]{\gamma}||\qdist[t]{\gamma'}) \leq \mathrm{KL}(\qdist{\gamma}||\qdist{\gamma'}).$ But, we have assumed that at time $t,$ our mixture family is such that $\mathrm{KL}(\qdist{\gamma}||\qdist{\gamma'}) \leq \epsilon_t^2/4 \leq (1/128) (\gamma - \gamma')^2/A_t^2 \leq (1/128)\mathrm{KL}(\qdist[t]{\gamma}||\qdist[t]{\gamma'}).$ Thus, the oracle observations and the statistical experiment that can be performed with the model $\qdist[t]{\gamma}$ cannot be simulated from having access to the oracle, $\qdist{\gamma}.$ 
\end{proof}

Before we give the proof of Theorem \ref{thm:sensitivity}, we recapitulate a corollary of Generalized Fano's method from Lemma 3 of \citep{binyu} and define covering entropy, which is used in Theorem \ref{thm:sensitivity}.
\begin{lemma}{[Lack of identifiability of mixture weights]} Let $G:= \{\gamma_1, \cdots, \gamma_N\} \in (0,1)^N$ be a set of $N\geq 16$ values of mixture weights such that $|\gamma_i - \gamma_j| \geq \varepsilon,\: \: i \neq j, \; 1\leq i,j \leq N.$ Then, for the class of score-restricted estimators (\cref{def:scoreEstimator}), $\Gamma_{\mathrm{score},0},$ with oracle access to the model, $\qdist{\gamma^*}$,   if $\mathrm{KL}(\qdist{\gamma_i}||\qdist{\gamma_j}) \leq \log N/4,$ we have
\begin{align}
\inf_{\hat{\gamma} \in \Gamma_{\mathrm{score},0}} \max_{\gamma^* \in G} \mathbb{E}_{\qdist{\gamma^*}} [|\hat{\gamma} - \gamma^*|^2] \geq \varepsilon^2/8.
\label{eq:fano}
\end{align}
\label{lem:fano}
\end{lemma}

We note that the proof follows from \citep[Lemma 3]{binyu}, which is an application of Fano's inequality. To get the error in the squared norm, $\mathbb{E}[|\hat{\gamma} - \gamma^*|^2,$ as we have in \eqref{eq:fano}, we need to additionally use Markov's inequality: $\mathbb{E}[|\hat{\gamma} - \gamma|^2 \geq (\varepsilon/2)^2 \mathbb{P}(|\hat{\gamma} - \gamma^*| > \varepsilon/2) \geq (\varepsilon/2)^2 \: (1/2).$ The second inequality directly follows from the proof of Lemma 3 of \citep{binyu}. The main difference here is that the data used by the estimator corresponds to the score (\cref{def:scoreEstimator}), chosen to exemplify the diffusion model setting.

\begin{definition}{[$\epsilon$-covering entropy]}\label{def:cover-entropy}
Let $O_t$ be a finite open cover of $[0,1],$ with centers $\gamma_1, \cdots, \gamma_{N_t},$ such that for any $\gamma \in (0,1),$ $\min_j \mathrm{KL}(\qdist[t]{\gamma}||\qdist[t]{\gamma_j}) < \epsilon^2.$ Then, $V_{\epsilon,t} := \log |O_t|$ is called the $\epsilon$-covering entropy. 
    
\end{definition}

The main novelty in the proof of Theorem \ref{thm:sensitivity} is the setup of score oracle-based estimators. Note that the above standard definitions of packing and covering entropies extend verbatim to our infinite-dimensional statistical models.

Secondly, while annealing has been known to be statistically beneficial \citep{koehler2022statisticalefficiencyscorematching, qin2024fit}, our result proves better minimax bounds for parameter recovery, rather than density estimation. 
To complete the upper bound in \eqref{eq:upperBound}, we apply directly Theorem 2 \citep{yangBarron}. To obtain the lower bound, we use Lemma \ref{lem:fano} with a possibly different packing set of size $N(\gamma^*)$ around each $\gamma^*.$ Using the definition of $\epsilon$-sensitivity and assuming that $\log N(\gamma^*) \geq \max\{\log 16, \sigma_t^2\: V(\epsilon_{t}) \}$ at all $\gamma^*,$ we obtain that for any $\gamma, \gamma' \in S(\gamma^*),$ $\mathrm{KL}(\qdist{\gamma}||\qdist{\gamma'}) \leq \delta = \epsilon_{t}^2/(4 \sigma_t^2) \leq (\log N(\gamma^*)/V(\epsilon_{t}))  \epsilon_{t}^2/4 = \log N(\gamma^*)/4,$ since $\epsilon_{t}^2 \: \sigma_t^2 = V(\epsilon_{t}).$ As a result, the conditions of Lemma \ref{lem:fano} apply at every $\gamma^*.$

\section{Diffusion process on Gaussian mixture models}
\subsection{Mixture Family Preliminaries}

Let $\pdist{0}$ and $\pdist{1}$ be two fixed probability distributions on $\mathbb{R}^d$. We recall the family of mixed probability densities
\begin{equation*}
    \mathcal{P}(\pdist{0},\pdist{1}) \;\coloneqq\; \big\{\, \pdist{\gamma} \coloneqq (1-\gamma)\,\pdist{0}({\vx}) + \gamma\,\pdist{1}({\vx}) \;\big|\; \gamma \in [0,1] \,\big\}.
\end{equation*}
The set of noisy marginals  at noise level $t$ produced by the forward process \eqref{eq:forward-sde} applied to samples ${\vx}_0 \sim \pdist{\gamma}$ is denoted $\{\pdist[t]{\gamma}\}_{t\in [0,1]}$. 
The forward process preserves the mixture structure: the noisy marginal $\pdist[t]{\gamma}$ is itself a mixture of the noisy component marginals $\pdist[t]{0}$ and $\pdist[t]{1}$, with the same mixture weight $\gamma$. The family of noisy mixtures and the noisy family of mixtures therefore coincide.
\begin{lemma}[Mixture closure under forward VP-SDE]
\label{lemma:mixture-closure}
For every $\gamma \in (0,1)$ and every $t \in [0,1]$, 
\begin{equation}
    \pdist[t]{\gamma} ({\vx}) = (1-\gamma) \pdist[t]{0}({\vx}) + \gamma \pdist[t]{1}({\vx}) 
    \label{eq:mixture-closure}
\end{equation}
\end{lemma}
\begin{proof}
Let $k_t(\vect{y} \mid {\vx}) := \CalN(\vect{y};\, \sqrt{\alpha_t}\,{\vx},\, (1-\alpha_t)\,\vect{I})$ denote the forward conditional density~\eqref{eq:forward-kernel}. Since $k_t$ does not depend on the initial distribution, the marginal at noise level $t$ is obtained by a linear convolution against the initial density:
\[
    \pdist[t]{\gamma}({\vx}) \;=\; \int_{\RR^d} k_t({\vx} \mid {\vx}_0)\,\pdist{\gamma}({\vx}_0)\,\mathrm{d}{\vx}_0.
\]
Substituting the definition of the mixture and applying the linearity of integration,
\begin{align*}
    \pdist[t]{\gamma}({\vx})
    \;&=\; (1-\gamma) \int_{\RR^d} k_t({\vx} \mid {\vx}_0)\,\pdist{0}({\vx}_0)\,\mathrm{d}{\vx}_0
    \;+\; \gamma \int_{\RR^d} k_t({\vx} \mid {\vx}_0)\,\pdist{1}({\vx}_0)\,\mathrm{d}{\vx}_0
    \\
    \;&=\; (1-\gamma)\,\pdist[t]{0}({\vx}) + \gamma\,\pdist[t]{1}({\vx}).
\end{align*}
\end{proof}

A useful consequence of \Cref{lemma:mixture-closure} is that samples from $\pdist[t]{\gamma}$ can be drawn by first sampling a component label $c \sim \mathrm{Ber}(\gamma)$ and then drawing $\vect{y}_t \sim \pdist[t]{c}$, a component-wise decomposition that simplifies subsequent analysis.

\subsubsection{The score of the noisy mixture}

Differentiating $\log \pdist[t]{\gamma}$ and applying \Cref{lemma:mixture-closure} gives the score of the noisy mixture as a pointwise convex combination of the component scores,
\begin{equation}
    \nabla_{\vx} \log \pdist[t]{\gamma}({\vx})
    \;=\; w_0^t({\vx})\,\nabla_{\vx} \log \pdist[t]{0}({\vx}) \;+\; w_1^t({\vx})\,\nabla_{\vx} \log \pdist[t]{1}({\vx}),
    \label{eq:mixture-score-app}
\end{equation}
where the weights are the EM responsibilities --- the posterior probabilities under the generative model that ${\vx}$ was produced by component $0$ or component $1$, respectively \cite[\S 9.2]{10.5555/1162264},
\begin{equation}
    w_0^t({\vx}) \;\coloneqq\; \frac{(1-\gamma)\,\pdist[t]{0}({\vx})}{\pdist[t]{\gamma}({\vx})}, \qquad
    w_1^t({\vx}) \;\coloneqq\; \frac{\gamma\,\pdist[t]{1}({\vx})}{\pdist[t]{\gamma}({\vx})},
    \label{eq:mixture-weights-app}
\end{equation}
and satisfy $w_0^t({\vx}) + w_1^t({\vx}) = 1$.

\subsubsection{Specialization: Gaussian components}

When the components are Gaussian, $\pdist{c} = \mathcal{N}(\boldsymbol{\mu}_c^*, \boldsymbol{\Sigma}_c^*)$ for $c \in \{0,1\}$, the noisy marginals remain Gaussian, $\pdist[t]{c} = \mathcal{N}(\boldsymbol{\mu}_{c,t}^*, \boldsymbol{\Sigma}_{c,t}^*)$, and the score \eqref{eq:mixture-score-app} takes the explicit form
\begin{equation}
    \nabla_{\vx} \log \pdist[t]{\gamma}({\vx}) \;=\; \sum_{c \in \{0,1\}} w_c^t({\vx})\, (\boldsymbol{\Sigma}_{c,t}^*)^{-1} (\boldsymbol{\mu}_{c,t}^* - {\vx}),
    \label{eq:gaussian-mixture-score}
\end{equation}
a responsibility-weighted sum of precision-scaled displacements $(\boldsymbol{\Sigma}_{c,t}^*)^{-1}(\boldsymbol{\mu}_{c,t}^* - {\vx})$. The responsibility weights \eqref{eq:mixture-weights-app} admit the explicit closed form
\begin{equation}
    w_c^t({\vx}) \;=\; \frac{\pi_c\, \mathcal{N}\!\left({\vx};\, \boldsymbol{\mu}_{c,t}^*,\, \boldsymbol{\Sigma}_{c,t}^*\right)}{\sum_{c' \in \{0,1\}} \pi_{c'}\, \mathcal{N}\!\left({\vx};\, \boldsymbol{\mu}_{c',t}^*,\, \boldsymbol{\Sigma}_{c',t}^*\right)},
    \label{eq:responsibility-weights}
\end{equation}
with mixture weights $\pi_0 = 1-\gamma$ and $\pi_1 = \gamma$. The numerator and denominator together carry information about $\gamma$, the Mahalanobis distance from ${\vx}$ to each component mean, and the Gaussian normalizing constants $\det(\boldsymbol{\Sigma}_{c,t}^*)^{-1/2}$.

\subsubsection{DDPM-style forward process} 

\begin{lemma}\label{lemma:mix_over_time}
For $\pdist{0}:=\mathcal{N}(\vx;\mu_0,\vect I)$ and $\pdist{1}:=\mathcal{N}(\vx;\mu_1,\vect I)$ denote two individual component distributions. For a variance-preserving forward process with noise schedule $(\bar{\alpha}_t)_{t \in [0,T]}$, the parameters of the mixture modes at time $t$ are
\begin{equation}
    \boldsymbol{\mu}_{c,t}^* \;=\; \sqrt{\bar{\alpha}_t}\, \boldsymbol{\mu}_c^*, \qquad
    \boldsymbol{\Sigma}_{c,t}^* \;=\; \bar{\alpha}_t\, \boldsymbol{\Sigma}_c^* + (1 - \bar{\alpha}_t)\, \mathbf{I},
    \label{eq:ddpm-perturbation}
\end{equation}
Hence, the noisy marginal distributions are,     
    \begin{align*}
        \pdist[t]{0}(\vx) = \mathcal{N}(\vx;\sqrt{\alpha_t} \mu_0, \vect I),
          \qquad 
        \pdist[t]{1}(\vx) = \mathcal{N}(\vx;\sqrt{\alpha_t} \mu_1, \vect I)
    \end{align*}
\end{lemma}

\begin{proof}
    Let $\vx_0\sim \mathcal{N}(\mu,\vect \Sigma)$. We have that 
    \begin{align*}
        \vx_t | \vx_0 &\sim \mathcal{N}(\sqrt{\alpha_t} \vx_0, (1-\alpha_t) \vect I)
        \\
        \Longrightarrow \vx_t &= \sqrt{\alpha_t}\vx_0 + \sqrt{1-\alpha_t} \vect z
    \end{align*}
    where $\vect z\sim \mathcal{N}(0,\vect I)$ is independent of $\vx_0$. Therefore,
    \begin{align*}
        \vx_t &\sim \mathcal{N}\fitparenth{\sqrt{\alpha_t} \vect \mu, \alpha_t \vect \Sigma +(1-\alpha_t)\vect I}
    \end{align*}
\end{proof}

As per \eqref{eq:ddpm-perturbation} $\alpha_t$ decreases from $1$ to $0$, the component means contract toward the origin and the covariances inflate toward $\mathbf{I}$. Substituting \eqref{eq:ddpm-perturbation} into \eqref{eq:responsibility-weights} reveals two limits of interest. At small $t$ ($\alpha_t \to 1$), the perturbed components remain well separated, and the responsibilities concentrate sharply on the closest mode, recovering the M-step regime of the EM algorithm. At large $t$ ($\alpha_t \to 0$), the perturbed components all approach $\mathcal{N}(\mathbf{0}, \mathbf{I})$ and the responsibilities flatten toward the prior weights $(1-\gamma, \gamma)$, becoming nearly independent of ${\vx}$. Sensitivity of the score to the mixture parameter $\gamma$ is therefore concentrated at the noise scales for which the responsibility weights are neither saturated nor uninformative --- a window whose location and width are determined by the geometry of the components.

We now demonstrate an equivalence between the neural network loss \eqref{eqn: neural network loss} and the diffusion score matching loss \eqref{eq:diffScoreMatchingLoss}. 
\begin{theorem}
\label{lem:nn-dsm-ordering}
Let $\pdist{\gamma}$ be a data distribution on $\mathbb{R}^d$ with $\mathbb{E}_{\pdist{\gamma}}\|\vx_0\|^2 < \infty$, and let $\{\pdist[t]{\gamma}\}_{t\in[0,1]}$ be its marginals under the VP-SDE \eqref{eq:forward-sde}--\eqref{eq:forward-kernel}. Let $\vs_\btheta:\mathbb{R}^d \times [0,1] \to \mathbb{R}^d$ be a score model with $\mathbb{E}_{\pdist[t]{\gamma}}\|\vs_\btheta(\cdot,t)\|^2 < \infty$ for almost every $t \in [0,1]$. Then
\begin{equation}
   \lnn(\pdist{\btheta}, \pdist{\gamma}) \;=\; \ldsm(\pdist{\btheta}, \pdist{\gamma}) \;+\; C(\pdist{\gamma}),
   \label{eq:nn-dsm-equivalence}
\end{equation}
where the constant $C(\pdist{\gamma})$ does not depend on $\btheta$ and admits the closed form
\begin{align}   \label{eq:nn-dsm-constant}
   C(\pdist{\gamma}) &\;=\; \int_0^1 \lambda_t \cdot \frac{\alpha_t}{(1-\alpha_t)^2}\,\mathrm{MMSE}_t(\pdist{\gamma})\, dt,\\
\mathrm{MMSE}_t(\pdist{\gamma}) 
&\;:=\; 
\mathbb{E}_{\substack{\vx_0 \,\sim\, \pdist{\gamma} \\ \vx_t \,\sim\, \pdist[t]{\gamma}(\cdot\mid\vx_0)}}\!\left[\,\big\|\vx_0 - \widehat{\vx}_0(\vx_t)\big\|^2\,\right],\\
\widehat{\vx}_0(\vx_t) &\;:=\; \mathbb{E}_{\vx_0' \,\sim\, \pdist{\gamma}(\cdot\mid\vx_t)}[\vx_0'].
\end{align}
Moreover, $C(\pdist{\gamma}) \ge 0$, with equality if and only if $\pdist{\gamma}$ is a Dirac mass. In particular,
\begin{equation}
   \lnn(\pdist{\btheta}, \pdist{\gamma}) \;\ge\; \ldsm(\pdist{\btheta}, \pdist{\gamma})
   \label{eq:nn-dsm-ordering}
\end{equation}
for every $\btheta \in \boldsymbol{\Theta}$.
\end{theorem}

\begin{proof}
Our arguments follow the reasoning in \cite{6795935} with an additional logic for guaranteeing the required relation between the neural network loss and the diffusion score matching loss.  
Throughout, all expectations are taken under $\pdist{\gamma}(\vx_0)$, $\pdist[t]{\gamma}(\vx_t\mid \vx_0)$ for the joint, and under $\pdist[t]{\gamma}(\vx_t)$ for the marginal. We suppress the dependence on $t$ in $\vs_\btheta(\vx_t, t)$ when convenient.
Fix $t \in [0,1]$. Expanding the squared norm in the inner expectation of $\ldsm$,
\begin{equation}
   \mathbb{E}_{\vx_t}\!\big\|\vs_\btheta(\vx_t,t) - \score[t]{\gamma}(\vx_t)\big\|^2
   = \mathbb{E}_{\vx_t}\|\vs_\btheta(\vx_t,t)\|^2 \,-\, 2\,b_t(\btheta) \,+\, \mathbb{E}_{\vx_t}\!\big\|\score[t]{\gamma}(\vx_t)\big\|^2,
   \label{eq:proof-esm-expand}
\end{equation}
with $b_t(\btheta) := \mathbb{E}_{\vx_t}\!\big\langle \vs_\btheta(\vx_t,t),\, \score[t]{\gamma}(\vx_t)\big\rangle$. Similarly,
\begin{align}\label{eq:proof-dsm-expand}
&\mathbb{E}_{\vx_0,\vx_t}\!\big\|\vs_\btheta(\vx_t,t) - \nabla_{\vx_t}\log \pdist[t]{\gamma}(\vx_t \mid \vx_0)\big\|^2\\\nonumber 
   &\qquad = \mathbb{E}_{\vx_0,\vx_t}\|\vs_\btheta(\vx_t, t)\|^2 \,-\, 2\,a_t(\btheta) \,+\, \mathbb{E}_{\vx_0,\vx_t}\!\big\|\nabla_{\vx_t}\log \pdist[t]{\gamma}(\vx_t\mid\vx_0)\big\|^2,
\end{align}
with $a_t(\btheta) := \mathbb{E}_{\vx_0,\vx_t}\!\big\langle \vs_\btheta(\vx_t,t),\, \nabla_{\vx_t}\log \pdist[t]{\gamma}(\vx_t\mid\vx_0)\big\rangle$. 
The first term in each expansion coincides because $\vs_\btheta(\vx_t,t)$ depends only on $\vx_t$. We now claim that for all $\theta$, $b_t(\btheta) = a_t(\btheta)$. Using $\score[t]{\gamma} = \nabla \pdist[t]{\gamma}/\pdist[t]{\gamma}$ and the marginalization $\pdist[t]{\gamma}(\vx_t) = \int \pdist[t]{\gamma}(\vx_t\mid\vx_0)\,\pdist{\gamma}(\vx_0)\, d\vx_0$,
\begin{align*}
   b_t(\btheta) 
   &= \int \pdist[t]{\gamma}(\vx_t)\,\big\langle \vs_\btheta(\vx_t,t),\, \score[t]{\gamma}(\vx_t)
   \big\rangle\, d\vx_t \\
   &= \int \big\langle \vs_\btheta(\vx_t,t),\, \nabla_{\vx_t}\pdist[t]{\gamma}(\vx_t)\big\rangle\, d\vx_t \\
   &= \int \Big\langle \vs_\btheta(\vx_t,t),\, \nabla_{\vx_t}\!\!\int \pdist[t]{\gamma}(\vx_t\mid\vx_0)\,\pdist{\gamma}(\vx_0)\, d\vx_0\Big\rangle\, d\vx_t \\
   &= \iint \big\langle \vs_\btheta(\vx_t,t),\, \nabla_{\vx_t}\pdist[t]{\gamma}(\vx_t\mid\vx_0)\big\rangle\,\pdist{\gamma}(\vx_0)\, d\vx_0\, d\vx_t \\
   &= \iint \pdist[t]{\gamma}(\vx_t\mid\vx_0)\,\pdist{\gamma}(\vx_0)\,\big\langle \vs_\btheta(\vx_t,t),\, \nabla_{\vx_t}\log \pdist[t]{\gamma}(\vx_t\mid\vx_0)\big\rangle\, d\vx_0\, d\vx_t \\
   &= a_t(\btheta),
\end{align*}
where Leibniz's rule applies by smoothness and Gaussian decay of the forward kernel, and Fubini's theorem applies due to the regularity assumption on $\vs_{\btheta}$. Combining the above,
\begin{equation}
   \mathbb{E}_{\vx_0,\vx_t}\!\big\|\vs_\btheta(\vx_t, t) - \nabla_{\vx_t}\log \pdist[t]{\gamma}(\vx_t\mid\vx_0)\big\|^2
   - \mathbb{E}_{\vx_t}\!\big\|\vs_\btheta(\vx_t, t) - \score[t]{\gamma}(\vx_t) \big\|^2
   = \mathrm{gap}(t),
   \label{eq:proof-residual}
\end{equation}
where
\begin{equation}
   \mathrm{gap}(t) := \mathbb{E}_{\vx_0,\vx_t}\!\big\|\nabla_{\vx_t}\log \pdist[t]{\gamma}(\vx_t\mid\vx_0)\big\|^2 - \mathbb{E}_{\vx_t}\!\big\|\score[t]{\gamma}(\vx_t)\big\|^2,
   \label{eq:proof-Dt-def}
\end{equation}
which is independent of $\btheta$. The Gaussian forward kernel \eqref{eq:forward-kernel} gives
\begin{equation}
   \nabla_{\vx_t}\log \pdist[t]{\gamma}(\vx_t\mid\vx_0)
   = -\frac{\vx_t - \sqrt{\alpha_t}\vx_0}{1-\alpha_t}
   = -\frac{\vect{z}}{\sqrt{1-\alpha_t}},\quad \vect{z}\sim\mathcal{N}(\vect{0},\vect{I}),
   \label{eq:proof-cond-score}
\end{equation}
hence the first term of \eqref{eq:proof-Dt-def} equals
\begin{equation}
\mathbb{E}~\!\big\|\nabla_{\vx_t}\log \pdist[t]{\gamma}(\vx_t\mid\vx_0)\big\|^2 = \frac{d}{1-\alpha_t}.
   \label{eq:proof-cond-norm}
\end{equation}
For the marginal score, take expectations of \eqref{eq:proof-cond-score} conditional on $\vx_t$ we obtain Tweedie's formula:
\begin{equation}
   \score[t]{\gamma}(\vx_t)
   = \mathbb{E}~\!\big[\nabla_{\vx_t}\log \pdist[t]{\gamma}(\vx_t\mid\vx_0)\,\big|\,\vx_t\big]
   = \frac{\sqrt{\alpha_t}\,\mathbb{E}[\vx_0\mid\vx_t] - \vx_t}{1-\alpha_t}.
   \label{eq:proof-tweedie}
\end{equation}
Setting $\hat{\bm{\eta}}(\vx_t) := \sqrt{\alpha_t}\,\mathbb{E}[\vx_0\mid\vx_t] = \mathbb{E}[\sqrt{\alpha_t}\vx_0 \mid \vx_t]$, we see that, 
\begin{equation}
   \mathbb{E}~\!\big\|\score[t]{\gamma}(\vx_t)\big\|^2 = \frac{\mathbb{E}\|\vx_t - \hat{\bm{\eta}}(\vx_t)\|^2}{(1-\alpha_t)^2}.
   \label{eq:proof-score-norm}
\end{equation}
Decomposing $\vx_t - \hat{\bm{\eta}} = (\vx_t - \sqrt{\alpha_t}\vx_0) + (\sqrt{\alpha_t}\vx_0 - \hat{\bm{\eta}}) = \sqrt{1-\alpha_t}\,\vect{z} + \sqrt{\alpha_t}(\vx_0 - \mathbb{E}[\vx_0\mid\vx_t])$ and expanding the expectation we get,
\begin{equation}
   \mathbb{E}\|\vx_t - \hat{\bm{\eta}}\|^2 = (1-\alpha_t)\,d + 2\sqrt{\alpha_t(1-\alpha_t)}\,\mathbb{E}~\!\big\langle\vect{z},\,\vx_0 - \mathbb{E}[\vx_0\mid\vx_t]\big\rangle + \alpha_t\,\mathrm{MMSE}_t(\pdist{\gamma}).
   \label{eq:proof-decomp}
\end{equation}
For the cross term, use $\vect{z} = (\vx_t - \sqrt{\alpha_t}\vx_0)/\sqrt{1-\alpha_t}$ and orthogonality of conditional expectation, $\mathbb{E}~\!\big[(\vx_0 - \mathbb{E}[\vx_0\mid\vx_t])\,\widehat{\vx}_0(\vx_t)\big] = 0$ for any $g$:
\begin{align*}
   \mathbb{E}~\!\big\langle\vect{z},\,\vx_0 - \mathbb{E}[\vx_0\mid\vx_t]\big\rangle
   &= \tfrac{1}{\sqrt{1-\alpha_t}}\,\mathbb{E}~\!\big\langle \vx_t - \sqrt{\alpha_t}\vx_0,\,\vx_0 - \mathbb{E}[\vx_0\mid\vx_t]\big\rangle \\
   &= \tfrac{1}{\sqrt{1-\alpha_t}}\,
   \widehat{\vx}_0(0 - \sqrt{\alpha_t}\,\mathrm{MMSE}_t(\pdist{\gamma})\big) \\
   &= -\frac{\sqrt{\alpha_t}\,\mathrm{MMSE}_t(\pdist{\gamma})}{\sqrt{1-\alpha_t}}.
\end{align*}
Substituting into \eqref{eq:proof-decomp},
\begin{equation}
   \mathbb{E}\|\vx_t - \hat{\bm{\eta}}\|^2 = (1-\alpha_t)\,d - \alpha_t\,\mathrm{MMSE}_t(\pdist{\gamma}),
\end{equation}
and therefore via \eqref{eq:proof-score-norm},
\begin{equation}
   \mathbb{E}~\!\big\|\score[t]{\gamma}(\vx_t)\big\|^2 = \frac{d}{1-\alpha_t} - \frac{\alpha_t\,\mathrm{MMSE}_t(\pdist{\gamma})}{(1-\alpha_t)^2}.
   \label{eq:proof-score-final}
\end{equation}
Combining \eqref{eq:proof-cond-norm} and \eqref{eq:proof-score-final} in \eqref{eq:proof-Dt-def},
\begin{equation}
   \mathrm{gap}(t) = \frac{\alpha_t\,\mathrm{MMSE}_t(\pdist{\gamma})}{(1-\alpha_t)^2} \;\ge\; 0.
   \label{eq:proof-Dt-final}
\end{equation}
Hence, multiplying \eqref{eq:proof-residual} by $\lambda_t$ and integrating over $t\sim \mathrm{Unif}(0,1)$, the difference in the neural network loss and the diffusion score matching loss can be written as, 
\begin{equation}
   \lnn(\pdist{\btheta}, \pdist{\gamma}) - \ldsm(\pdist{\btheta}, \pdist{\gamma})
   = \int_0^1 \lambda_t\,\mathrm{gap}(t)\,dt
   = \int_0^1 \lambda_t\,\frac{\alpha_t\,\mathrm{MMSE}_t(\pdist{\gamma})}{(1-\alpha_t)^2}\,dt
   =: C(\pdist{\gamma}),
\end{equation}
proving \eqref{eq:nn-dsm-equivalence}--\eqref{eq:nn-dsm-constant}.
For the equality case, note that $C(\pdist{\gamma}) = 0$ iff $\mathrm{MMSE}_t(\pdist{\gamma}) = 0$ for almost every $t \in (0,1)$. The latter requires $\vx_0 = \mathbb{E}[\vx_0\mid\vx_t]$ almost surely, i.e., $\vx_0$ is $\sigma(\vx_t)$-measurable. Since $\vx_t = \sqrt{\alpha_t}\vx_0 + \sqrt{1-\alpha_t}\vect{z}$ with $\vect{z}$ independent of $\vx_0$ and $1-\alpha_t > 0$, this forces $\vx_0$ to be deterministic.
\end{proof}

\subsection{Proof of  \Cref{thm: GMM Sensitivity Emerges}}\label{Proof of Theorem 1}

\begin{remark}
    In this theorem, we use $\lambda_t=1$ rather than $\lambda_t= 1-\alpha_t$. 
    The value of $t$ where $\lsm(\pdist[t]{\gamma},\pdist[t]{\gamma^*})$ is large correspond to when $\alpha_t$ is close to $0$ for when the modes are reasonably well-separated, so that $1-\alpha_t$ is close to $1$. The same analysis can be performed in the presence of any weighting function $\lambda_t$ if one is willing to accept a more complicated expression for the constant $C$; we elaborate on this in \Cref{prop: Lower bound integral}. For all of our numerical results, we use the convention that $\lambda_t=1-\alpha_t$ to show that our results still hold for the weighting factor used in practice.
\end{remark}

Throughout the proof, let $\separation:=\norm{\mu_1-\mu_0}$. From \Cref{lemma:mixture-closure} and \Cref{lemma:mix_over_time}, we can generate samples from $\pdist[t]{\gamma}$ by sampling $c\sim \mathrm{Ber}(\gamma)$ and sampling $\vect y_t\sim \pdist[t]{c}$. Furthermore, the next lemma will show that we can generate these samples using the Probability Flow ODE (\citep{song2021scorebasedgenerativemodelingstochastic}) instead of the stochastic forward process.

First, note that 
\begin{align}\label{eq: mode decomposition}
    \ldsm(\pdist{\hat\gamma},\pdist{\gamma^*})&= \expect[t,\vect x_t\sim \pdist[t]{\gamma^*}]{\norm{\score[t]{\gamma^*}(\vect x_t)-\score[t]{\hat\gamma}}(\vect x_t)}
    \\
    &= (1-\gamma^*)\expect[t,\vect x_t\sim \pdist[t]{\gamma^*}|c=0]{\norm{\score[t]{\gamma^*}(\vect x_t)-\score[t]{\hat\gamma}}(\vect x_t)}
    \\
    &\quad\quad + \gamma^*\expect[t,\vect x_t\sim \pdist[t]{\gamma^*}|c=1]{\norm{\score[t]{\gamma^*}(\vect x_t)-\score[t]{\hat\gamma}}(\vect x_t)}
    \\
    &= (1-\gamma^*)\expect[t,\vect x_t\sim \pdist[t]{0}]{\norm{\score[t]{\gamma^*}(\vect x_t)-\score[t]{\hat\gamma}}(\vect x_t)}
    \\
    &\quad\quad + \gamma^*\expect[t,\vect x_t\sim \pdist[t]{1}]{\norm{\score[t]{\gamma^*}(\vect x_t)-\score[t]{\hat\gamma}}(\vect x_t)}
\end{align}

In particular, we can lower-bound $\ldsm$ by independently lower-bounding each of these expectations.

\begin{lemma}
    Consider any distribution $\pi$ and let $\pi_t$ be the noisy marginal at time $t$ of the forward process. Then, if $\vect y_t$ satisfies the Probability Flow ODE
    \begin{equation}\label{eqn: Probability Flow ODE}
        \mathrm{d}\vect y_t = -\frac{1}{2}\beta_t\fitparenth{\vect y_t + \nabla \log p_{\vect y_t}(\vect y_t)}\mathrm{d}t
    \end{equation}
    and $\vect y_0\sim \pi$, then $\vect y_t\sim \pi_t$.
\end{lemma}

\begin{proof}
    Let $\vx_t$ satisfy the Variance Preserving SDE and $\vect y_t$ be as in the lemma statement. Then, we have the Fokker-Planck equations
    \begin{align*}
        \frac{\partial p_{\vx_t} }{\partial t}&= -\nabla \cdot \fitparenth{-\frac{1}{2}\beta_t \vx p_{\vx_t} } + \frac{1}{2}\beta_t\Delta p_{\vx_t} 
        \\
        &=\frac{1}{2}\beta_t  \nabla \cdot \fitparenth{\vx p_{\vx_t} } + \frac{1}{2}\beta_t\Delta p_{\vx_t} 
        \\
        \\
        \frac{\partial p_{\vect y_t} }{\partial t}
        &= \frac{1}{2}\beta_t \nabla \cdot \fitparenth{\vect y p_{\vect y_t} }  +\frac{1}{2}\beta_t \nabla \cdot \fitparenth{p_{\vect y_t}\nabla \log p_{\vect y_t} } 
        \\
        &= \frac{1}{2}\beta_t \nabla \cdot \fitparenth{\vect y p_{\vect y_t} }  +\frac{1}{2}\beta_t \Delta p_{\vect y_t}
    \end{align*}
\end{proof}

The Probability Flow ODE (\eqref{eqn: Probability Flow ODE}) substantially simplifies the analysis because of the following fact:
\begin{lemma}\label{lemma: simple paths}
    For $\vect y_0$, and $\vect y_t$ following the Probability Flow ODE associated with $c$ and starting from $\vect y_0\sim \pdist{c}$, we have that 
    \[\vect y_t - \sqrt{\alpha_t}\mu_c = \vect y_0 - \mu_c\]
\end{lemma}

\begin{proof}
    Note that $\frac{\mathrm d}{\mathrm d t}\sqrt{\alpha_t}=\frac{1}{2}\beta_t \sqrt{\alpha_t}$. Therefore,
    \begin{align*}
        \frac{\mathrm d}{\mathrm dt}(\vect y_t -\sqrt{\alpha_t}\mu_c) &=-\frac{1}{2}\beta_t \fitparenth{\vect y_t + \nabla \log \pdist[t]{c}(\vect y_t)} - \frac{1}{2}\beta_t \sqrt{\alpha_t} \mu_c
        \\
        &=-\frac{1}{2}\beta_t \fitparenth{\vect y_t + \sqrt{\alpha_t} \mu_c - \vect y_t - \sqrt{\alpha_t} \mu_c}
        \\
        &=0
    \end{align*}
\end{proof}

We now focus on the case where $c=0$ (the case of $c=1$ will follow via symmetry). Let $\vect z_t = \vect y_t|c=0$. Our first goal is to lower bound $\norm{\score[t]{\gamma^*}(\vect z_t) - \score[t]{\hat \gamma}(\vect z_t)}^2$ pointwise in $t$ and $\vect z_t$. We first show an analytic expression for this score difference, then identify a range of $t$'s along each path where this expression is sufficiently large. Finally, by integrating over initial conditions $\vect z_0$, we get a lower bound for $\expect[t,\vect z_t\sim \pdist[t]{0}]{\norm{\score[t]{\gamma^*}(\vect z_t)-\score[t]{\hat\gamma}(\vect z_t)}^2}$.
\begin{proposition}\label{prop:exact norm expr}
    Let $\hat \gamma, \gamma^*\in [0,1]$. Then, 
    \[\norm{\score[t]{\hat \gamma}(\vect z_t) - \score[t]{\gamma^{*}}(\vect z_t)}^2 = ({\hat\gamma}-\gamma^*)^2  \frac{\norm{\mu_0-\mu_1}^2 \alpha_te^{2f_0(\vect z_0, t)}}{(1-{\hat\gamma} + {\hat\gamma}e^{f_0(\vect z_0, t)})^2 (1-\gamma^* + \gamma^* e^{f_0(\vect z_0, t)})^2}\]
    where 
    \[f_0(\vect z_0, t):=-\frac{1}{2}\norm{\mu_1-\mu_0}^2\alpha_t + \inprod{\vect z_0 - \mu_0, \mu_1 - \mu_0}\sqrt{\alpha_t}\]
\end{proposition}

\begin{proof}
    We \[\pdist[t]{0}(\vx)=\frac{1}{Z_t}\exp\fitparenth{-\frac{1}{2}\norm{\vx-\sqrt{\alpha_t}\mu_0}^2}\]
     \[\pdist[t]{1}(\vx)=\frac{1}{Z_t}\exp\fitparenth{-\frac{1}{2}\norm{\vx-\sqrt{\alpha_t}\mu_1}^2}\]

    Expanding the log likelihoods, we have
    \begin{align*}
        \log \frac{p_1(\vect z_t)}{p_0(\vect z_t)}&= -\frac{1}{2} \norm{\mu_0-\mu_1}^2 \alpha_t + \inprod{\vect z_0 - \mu_0, \mu_1 - \mu_0}\sqrt{\alpha_t}
    \end{align*}

    Let $f(\vect z_0,t) :=\log \frac{p_1(\vect z_t)}{p_0(\vect z_t)}$. Then,
    \begin{align*}
        \score[t]{\gamma}(\vect z_t)&= \frac{(1-\gamma)\;\pdist[t]{0}(\vect z_t)\score[t]{0}(\vect z_t) + \gamma\;\pdist[t]{1}(\vect z_t)\score[t]{1}(\vect z_t)}{\pdist[t]{\gamma}(\vect z_t)}
        \\
        &=\score[t]{0}(\vect z_t)  +  \frac{  \gamma\;\pdist[t]{1}(\vect z_t)\fitparenth{\score[t]{1}(\vect z_t)-\score[t]{0}(\vect z_t)}}{\pdist[t]{\gamma}(\vect z_t)}
        \\
        &=\score[t]{0}(\vect z_t)  +  \frac{  \gamma\;\pdist[t]{1}(\vect z_t)}{\pdist[t]{\gamma}(\vect z_t)}\sqrt{\alpha_t}\fitparenth{ \mu_1- \mu_0}
        \\
        &=\score[t]{0}(\vect z_t)  +  \frac{  \gamma\;e^{f(\vect z_0,t)}}{1-\gamma+\gamma \; e^{f(\vect z_0,t)}}\sqrt{\alpha_t}\fitparenth{ \mu_1- \mu_0}
    \end{align*}
    where we used $\nabla \log \pdist[t]{c}(\vx) = \sqrt{\alpha_t}\mu_c-\vx$. Therefore,
    \begin{align*}
        &\score[t]{\hat \gamma}(\vect z_t) - \score[t]{\gamma^{*}}(\vect z_t)= \frac{  ({\hat\gamma}-\gamma^*)\;e^{f(\vect z_0,t)}}{\fitparenth{1-{\hat\gamma}+{\hat\gamma} \; e^{f(\vect z_0,t)}}\fitparenth{1-\gamma^*+\gamma^* \; e^{f(\vect z_0,t)}}}\sqrt{\alpha_t}\fitparenth{ \mu_1- \mu_0}
    \end{align*}

    This gives
    \[\norm{\score[t]{\hat \gamma}(\vect z_t) -\score[t]{\gamma^{*}}(\vect z_t)}^2=({\hat\gamma}-\gamma^*)^2\frac{ \norm{ \mu_1- \mu_0}^2\alpha_t\;e^{2f(\vect z_0,t)}}{\fitparenth{1-{\hat\gamma}+{\hat\gamma} \; e^{f(\vect z_0,t)}}^2\fitparenth{1-\gamma^*+\gamma^* \; e^{f(\vect z_0,t)}}^2} \]
\end{proof}

\begin{lemma}
    For all $x\in \RR$ and $\gamma\in [0,1]$, we have 
    \[\frac{e^x}{(1-\gamma + \gamma e^x)^2}\geq \frac{1}{4\gamma (1-\gamma)}e^{-\frac{1}{4}\fitparenth{x-\log \frac{1-\gamma}{\gamma}}^2}\]
    with equality only when $x=\log \frac{1-\gamma}{\gamma}$. In particular, 
    \begin{align*}
        &\frac{e^{2f_0(\vect z_0, t)}}{\fitparenth{1-{\hat\gamma} + {\hat\gamma} e^{f_0(\vect z_0, t)}}^2 \fitparenth{1-\gamma^* + \gamma^* e^{f_0(\vect z_0, t)}}^2}\geq \frac{e^{-\frac{1}{4}\fitparenth{f_0(\vect z_0,t)-\log \frac{1-{\hat\gamma}}{{\hat\gamma}}}^2-\frac{1}{4}\fitparenth{f_0(\vect z_0,t)-\log \frac{1-\gamma^*}{\gamma^*}}^2}}{16 {\hat\gamma} (1-{\hat\gamma}) \gamma^* (1-\gamma^*)}
    \end{align*}
\end{lemma}

\begin{proof}
    Let $g(x)= \frac{e^x}{(1-\gamma + \gamma e^x)^2}$ and $h(x)=\frac{1}{4\gamma (1-\gamma)}e^{-\frac{1}{4}\fitparenth{x-\log \frac{1-\gamma}{\gamma}}^2}$. Then, 
    \begin{align*}
        \frac{\mathrm d}{\mathrm d x}\log g(x) &=\frac{1-\gamma -\gamma e^x}{1-\gamma + \gamma e^x}
    \end{align*}

    For all $\gamma\in [0,1]$ $1-\gamma + \gamma e^x>0$. Therefore, $x^*:=\log \frac{1-\gamma}{\gamma}$ is the only critical point of $g(x)$, and $g(x^*)=\frac{1}{4\gamma\fitparenth{1-\gamma}}$. Similarly $h(x^*)=\frac{1}{4\gamma\fitparenth{1-\gamma}}$ and $\frac{\mathrm d}{\mathrm d x}\log h(x^*)=0$. We have
    \begin{align*}
        \frac{\mathrm{d}}{\mathrm{d}x}[\log h(x)-\log g(x)]
        &=-\frac{1}{2}(x-x^*) - \frac{1-e^{x-x^*}}{1+e^{x-x^*}}
        \\
        \frac{\mathrm{d}^2}{\mathrm{d}x^2}[\log h(x)-\log g(x)]
        &= -\frac{(1-e^{x-x^*})^2}{2(1+e^{x-x^*})^2}
    \end{align*}
    
    So, $\frac{\mathrm{d}^2}{\mathrm{d}x^2}[\log h(x)-\log g(x)]\leq 0$ with equality only when $x=x^*$. Therefore, $\log h-\log g$ is strictly concave away from $x^*$ and maximized at $x^*$, where $h(x^*)=g(x^*)$. Therefore, $h< g$.
\end{proof}

\begin{lemma}\label{lemma:bounds}
    Let $\vect z_t$ and $\vect z_0$ be as above. Let $M(\vect z_0):=\inprod{\vect z_0 - \mu_0, \mu_1 - \mu_0}$, and define the log-odds ratios 
    \begin{align*}
        \hat r &:= \log \frac{1-{\hat\gamma}}{{\hat\gamma}}
        \\
        r^* &:= \log \frac{1-\gamma^*}{\gamma^*}
    \end{align*}
    
    Let $\epsilon>0$ be such that 
    \begin{align*}
    \epsilon &\leq \exp\fitparenth{-\frac{1}{8}\fitbracket{(\hat r-r^*)^2 + \fitparenth{\hat r+r^*-\frac{M(\vect z_0)^2}{\separation^2}}_+^2}}
    \end{align*}
    where $(\cdot)_+:= \max\set{\cdot, 0}$.
    
    Let
    \begin{align*}
        a&=\max\left\{\frac{M(\vect z_0)}{\separation^2} + \frac{1}{\separation}\sqrt{\fitparenth{\frac{M(\vect z_0)^2}{\separation^2} - \fitparenth{ \hat r + r^*}-\sqrt{8\log \frac{1}{\epsilon}-\fitparenth{\hat r-r^*}^2}}_+},\sqrt{\alpha_1}\right\}
        \\
        b&= \min\left\{\frac{M(\vect z_0)}{\separation^2} + \frac{1}{\separation}\sqrt{\frac{M(\vect z_0)^2}{\separation^2} - \fitparenth{ \hat r + r^*}+\sqrt{8\log \frac{1}{\epsilon}-\fitparenth{\hat r-r^*}^2}},1\right\}
    \end{align*}

    Then for all $t$ such that $\sqrt{\alpha_t}\in [a,b]$, 
    \[e^{-\frac{1}{4}\fitparenth{f_0(\vect z_0,t)-\hat r }^s2-\frac{1}{4}\fitparenth{f_0(\vect z_0,t)-r^*}^2} \geq \epsilon\]
\end{lemma}

\begin{proof}
    First, the inequality is equivalent to 
    \begin{align*}
       \fitparenth{f_0(\vect z_0,t)-r^*}^2+\fitparenth{f_0(z_0,t)-r^*}^2\leq 4\log \frac{1}{\epsilon}
    \end{align*}

    Plugging in the definition of $f$, the inequality holds whenever 
    \[\fitbracket{\frac{\separation^2}{2} \fitparenth{\sqrt{\alpha_t}-\frac{M(\vect z_0)}{\separation^2}}^2 - \frac{M(\vect z_0)^2}{2\separation^2}+ \frac{\hat r+r^*}{2}}^2 \leq 2\log \frac{1}{\epsilon}-\frac{(\hat r-r^*)^2}{4} \]
    
    We wish to identify the values of $\sqrt{\alpha_t}$ which make this inequality hold. This function is continuous, so its sign can only change at the roots of the corresponding equality constraint:
    \[\fitbracket{\frac{\separation^2}{2} \fitparenth{\sqrt{\alpha_t}-\frac{M(\vect z_0)}{\separation^2}}^2 - \frac{M(\vect z_0)^2}{2\separation^2}+ \frac{\hat r+r^*}{2}}^2 = 2\log \frac{1}{\epsilon}-\frac{(\hat r -r^*)^2}{4}\]

    Solving this quartic gives the roots
    \[\sqrt{\alpha_t}=\frac{M(\vect z_0)}{\separation^2} \pm\sqrt{\frac{M(\vect z_0)^2}{\separation^4}-\frac{\hat r+r^*}{\separation^2}\pm \frac{1}{\separation^2}\sqrt{8\log \frac{1}{\epsilon}-(\hat r-r^*)^2}}\]
    
    The condition on $\epsilon$ is such that the square roots are all well-defined. Let $b'$ be the largest of these roots (both $\pm$'s are set to $+$). As $\sqrt{\alpha_t}\rightarrow \infty$, we would have $f_0(\vect z_0,t)\rightarrow -\infty$ (because $-\separation^2/2$ is strictly negative), and hence 
    \[e^{-\frac{1}{4}\fitparenth{f_0(\vect z_0,t)-\hat r }^2-\frac{1}{4}\fitparenth{f_0(\vect z_0,t)-r^*}^2} \rightarrow 0\]
    
    Therefore, for $\sqrt{\alpha_t}>b'$ the inequality does not hold. Note that the four roots above are distinct (with two being complex for $\epsilon$ sufficiently small), meaning that each root has multiplicity 1, and hence the polynomial changes signs at $b'$. Therefore, for $\sqrt{\alpha_t}$ between $b'$ and the next largest real root, the inequality holds. The next largest real root is
    \[a'=\begin{cases}
        \frac{M(\vect z_0)}{\separation^2} +\sqrt{\frac{M(\vect z_0)^2}{\separation^4}-\frac{\hat r+r^*}{\separation^2}-\frac{1}{\separation^2}\sqrt{8\log \frac{1}{\epsilon}-(\hat r-r^*)^2}} & \text{if}\quad \frac{M(\vect z_0)^2}{\separation^2}-(\hat r+r^*)-\sqrt{8\log \frac{1}{\epsilon}-(\hat r-r^*)^2} \geq0
        \\
        \frac{M(\vect z_0)}{\separation^2} -\sqrt{\frac{M(\vect z_0)^2}{\separation^4}-\frac{\hat r+r^*}{\separation^2}+\frac{1}{\separation^2}\sqrt{8\log \frac{1}{\epsilon}-(\hat r-r^*)^2}} & \text{otherwise}
    \end{cases}\]
    
    In the second case we have $a'\leq \frac{M(\vect z_0)}{\separation^2}$, so
    \[a'\leq \frac{M(\vect z_0)}{\separation^2} +\sqrt{\fitparenth{\frac{M(\vect z_0)^2}{\separation^4}-\frac{\hat r+r^*}{\separation^2}-\frac{1}{\separation^2}\sqrt{8\log \frac{1}{\epsilon}-(\hat r-r^*)^2}}_+}=:a'' \]
    
    Therefore, the inequality holds on $[a',b']$ and $[a'',b']\subseteq [a',b']$, so the inequality holds on $[a'',b']$. Note that for $t\in [0,1]$, $\alpha_t\in [\alpha_1, 1]$ where $\alpha_1 >0$ is small, so we let $b:= \min\set{b', 1}$ and $a:=\max\set{a,\sqrt{\alpha_1}}$.
\end{proof}

\begin{proposition}\label{prop: lower bound over t}
    For all $\epsilon$ satisfying the condition in  \Cref{lemma:bounds}, for all $t$ s.t. $\sqrt{\alpha_t}\in[a,b]$ (with $a$ and $b$ as defined in the previous lemma) we have
    \begin{align*}
        \norm{\score[t]{\hat \gamma}(\vect z_t) - \score[t]{\gamma^{*}}(\vect z_t)}^2\geq\frac{(\hat \gamma-\gamma^*)^2\norm{ \mu_1- \mu_0}^2 \alpha_t}{16 \hat \gamma (1-\hat \gamma) \gamma^* (1-\gamma^*)}\epsilon
    \end{align*}
    
\end{proposition}

\begin{proof}
    Combining the above two lemmas,
    \begin{align*}
        \norm{\score[t]{\hat \gamma}(\vect z_t) - \score[t]{\gamma^{*}}(\vect z_t)}^2 &\geq\frac{(\hat \gamma-\gamma^*)^2\norm{ \mu_1- \mu_0}^2 \alpha_t}{16 \hat \gamma (1-\hat \gamma) \gamma^* (1-\gamma^*)}e^{-\frac{1}{4}\fitparenth{f_0(\vect z_0,t)-\hat r}^2-\frac{1}{4}\fitparenth{f_0(\vect z_0,t)-r^*}^2}
        \\
        &\geq\frac{(\hat \gamma-\gamma^*)^2\norm{ \mu_1- \mu_0}^2 \alpha_t}{16 \hat \gamma (1-\hat \gamma) \gamma^* (1-\gamma^*)}\epsilon
    \end{align*}
\end{proof}

We now turn our attention to lower bounding 
\[\int_0^1\norm{\score[t]{\hat \gamma}(\vect z_t) - \score[t]{\gamma^{*}}(\vect z_t)}^2\mathrm{d}t\]
along a given path starting from $\vect {z _ 0} $. From the lemmas above, we have
\[\int_0^1  \norm{\nabla \log \pdist[t]{\hat \gamma}(\vect z_t) - \nabla \log \pdist[t]{\gamma^{*}}(\vect z_t)}^2\mathrm{d}t > \frac{({\hat \gamma}-\gamma^*)^2\norm{ \mu_1- \mu_0}^2}{16{\hat \gamma}(1-{\hat \gamma}) \gamma^* (1-\gamma^*)}\epsilon\int_{t_{b^2}}^{t_{a^2}}\alpha_t\mathrm{d}t\] 
where $t_{a^2}$ and $t_{b^2}$ are such that $\alpha_{t_{a^2}}={a^2}$ and $\alpha_{t_{b^2}}={b^2}$. Note $\alpha_t$ is a decreasing function of $t$, hence why the bounds of the integral are reversed.  If we let $\alpha^{-1}(\tau)$ be the inverse function of $\alpha_t$, such that $\alpha_{\alpha^{-1}(\tau)} = \tau$, then
\[\int_{t_{b^2}}^{t_{a^2}}\alpha_t\mathrm{d}t=2\int_{a}^b \frac{\tau}{\beta_{\alpha^{-1}(\tau^2)}}\mathrm{d}\tau\]

In particular, if we consider the linear noise schedule 
\[\beta_t:=\beta_1 t\]
then
\[\int_{t_{b^2}}^{t_{a^2}}\alpha_t\mathrm{d}t=\frac{1}{\sqrt{\beta_1}}\int_{a}^b \frac{\tau}{\sqrt{-\ln \tau}}\mathrm{d}\tau\]

\begin{proposition}\label{prop: Lower bound integral}
    Let $a$, $b$, $\separation$, $z_0$, $\hat r$, $r^*$, and $M$, be as above, and assume $\epsilon$ satisfies the bound from \Cref{lemma:bounds}. Define 
    \[R_0:=\sqrt{8\ln \frac{1}{\epsilon}-(\hat r-r^*)^2}-(\hat r+r^*)\]
    
    Then, when $\frac{M(\vect z_0)}{\separation^2}\geq \sqrt{\alpha_1}$ and
    \[\epsilon \geq \exp\fitparenth{-\frac{1}{8}\fitbracket{(\separation^2 - 2M(\vect z_0)+(\hat r+r^*))^2 + (\hat r-r^*)^2}}\]
    then
    \begin{align*}
        \int_{a}^b \frac{\tau}{\sqrt{-\ln \tau}}\mathrm{d}\tau&\geq \frac{\frac{M(\vect z_0)}{\separation}\sqrt{\frac{M(\vect z_0)^2}{\separation^2} + R_0}+\frac{M(\vect z_0)^2}{2\separation^2} + \frac{1}{2}R_0}{\separation^2\sqrt{\ln(\separation)- \ln\fitparenth{\frac{M(\vect z_0)}{\separation}+ \frac{1}{2}\sqrt{\frac{M(\vect z_0)^2}{\separation^2} + R_0}}}}
    \end{align*}
\end{proposition}

\begin{remark}
    This proposition requires $\lambda_t=1$, as the analysis is far more challenging when the integrand $g$ is nonconvex. If one wishes to analyze alternative weighting factors or, indeed, an alternative noise schedule, all the analysis up to this point applies. The difference is that one now needs to lower bound the more challenging expression
    $\int_{t_{b^2}}^{t_{a^2}}\alpha_t\lambda_t\mathrm{d}t$. For the linear noise schedule and $\lambda_t=1-\alpha_t$, this take the form $\frac{1}{\sqrt{\beta_1}}\int_{a}^b \frac{\tau(1-\tau^2)}{\sqrt{-\ln \tau}}\mathrm{d}\tau$.
\end{remark}

\begin{proof}
    Let \[g(t):= \frac{\tau}{\sqrt{-\ln \tau}}\]

    Note that 
    \begin{align*}
        g''(t)&= \frac{1}{2t(-\ln t)^{3/2}} + \frac{3}{4t(-\ln t)^{5/2}} >0
    \end{align*}

    So, $g(t)$ is convex, meaning 
    \[\int_a^b g(\tau)\mathrm{d}\tau \geq (b-a)g\fitparenth{\frac{a+b}{2}}\]

    The condition that $\frac{M(\vect z_0)}{\separation^2}\geq \sqrt{\alpha_1}$ guarantees $a\geq\sqrt{\alpha_1}$. Similarly, the condition on $\epsilon$ guarantees $b\leq 1$, so
    \begin{align*}
        b-a&=\frac{1}{\separation}\sqrt{\frac{M(\vect z_0)^2}{\separation^2} - \fitparenth{ \hat r + r^*}+\sqrt{8\log \frac{1}{\epsilon}-\fitparenth{\hat r-r^*}^2}}
        \\
        &= \frac{1}{\separation}\sqrt{\frac{M(\vect z_0)^2}{\separation^2} + R_0}
    \end{align*}

    We have 
    \begin{align*}
        \frac{a+b}{2}&= \frac{M(\vect z_0)}{\separation^2}+\frac{1}{2\separation}\sqrt{\frac{M(\vect z_0)^2}{\separation^2} + R_0}
    \end{align*}
    and so
    \begin{align*}
        \Longrightarrow g\fitparenth{\frac{a+b}{2}}&= \frac{\frac{M(\vect z_0)}{\separation}+\frac{1}{2}\sqrt{\frac{M(\vect z_0)^2}{\separation^2} + R_0}}{\separation\sqrt{\ln(\separation)- \ln\fitparenth{\frac{M(\vect z_0)}{\separation}+ \frac{1}{2}\sqrt{\frac{M(\vect z_0)^2}{\separation^2} + R_0}}}}
    \end{align*}

    Combining these, 
    \begin{align*}
        \int_a^b g(\tau)\mathrm{d}\tau &\geq \frac{\frac{M(\vect z_0)}{\separation}\sqrt{\frac{M(\vect z_0)^2}{\separation^2} + R_0}+\frac{M(\vect z_0)^2}{2\separation^2} + \frac{1}{2}R_0}{\separation^2\sqrt{\ln(\separation)- \ln\fitparenth{\frac{M(\vect z_0)}{\separation}+ \frac{1}{2}\sqrt{\frac{M(\vect z_0)^2}{\separation^2} + R_0}}}}
    \end{align*}
\end{proof}

Combining this lower bound with what we showed before, we have 
\begin{align*}
    &\int_0^1\norm{\nabla \log \pdist[t]{\hat \gamma}(\vect z_t) - \nabla \log \pdist[t]{\gamma^{*}}(\vect z_t)}^2 \mathrm{d}t
    \\
    &\quad\quad\geq\frac{(\hat \gamma-\gamma^*)^2 \epsilon}{16\sqrt{\beta_1}\;\hat  \gamma (1-\hat \gamma) \gamma^* (1-\gamma^*)} \frac{\frac{M(\vect z_0)}{\separation}\sqrt{\frac{M(\vect z_0)^2}{\separation^2} + R_0}+\frac{M(\vect z_0)^2}{2\separation^2} + \frac{1}{2}R_0}{\sqrt{\ln(\separation)- \ln\fitparenth{\frac{M(\vect z_0)}{\separation}+ \frac{1}{2}\sqrt{\frac{M(\vect z_0)^2}{\separation^2} + R_0}}}}
\end{align*}

Recall that $\vect z_0\sim \mathcal{N}(\mu_0, \vect I)$, so 
\begin{align*}
    M(\vect z_0) &= \inprod{\vect z_0 -\mu_0, \mu_1-\mu_0} \sim \mathcal{N}(0, \norm{\mu_1-\mu_0}^2)
    \\
    \Longrightarrow \frac{M(\vect z_0)}{\separation}&\sim \mathcal{N}(0,1)
\end{align*}

\begin{proposition}\label{prop:c0_lowerbound}
    Let $\Phi$ be the CDF of $\mathcal{N}(0,1)$. Let $\tau\geq \separation \sqrt{\alpha_1}$, and let $\epsilon>0$ satisfy $ \epsilon \in (\epsilon_{\mathrm{min}}^0, \epsilon_{\mathrm{max}}^0)$ with 
    \begin{align*}
    \epsilon_{\mathrm{max}} ^0& =\exp\fitparenth{-\frac{1}{8}\fitbracket{(\hat r-r^*)^2 + \fitparenth{\hat r+r^*-\tau^2}_+^2}}
    \\
    \epsilon_{\mathrm{min}}^0 &= \exp\fitparenth{-\frac{1}{8}\fitbracket{(\separation(\separation - 2\tau)+(\hat r+r^*))^2 + (\hat r-r^*)^2}}
    \end{align*}
    
    Then, with probability $1-\Phi(\tau)$ over samples of $\vect z_0$, 
    \begin{align*}
        &\int_0^1\norm{\nabla \log \pdist[t]{\hat \gamma}(\vect z_t) - \nabla \log \pdist[t]{\gamma^{*}}(\vect z_t)}^2 \mathrm{d}t
        \\
        &\quad\quad\geq\frac{(\hat \gamma-\gamma^*)^2 \epsilon}{16\sqrt{\beta_1}\; \hat \gamma (1-\hat \gamma) \gamma^* (1-\gamma^*)} 
        \frac{\tau\sqrt{\tau^2 + R_0}+\frac{1}{2}\tau^2 + \frac{1}{2}R_0}{\sqrt{\ln(\separation)- \ln\fitparenth{\tau+ \frac{1}{2}\sqrt{\tau^2 + R_0}}}}
    \end{align*}
\end{proposition}

\begin{proof}
If we let $K:= \frac{M(\vect z_0)}{\separation}$ our bound becomes 
\begin{align*}
    &\int_0^1\norm{\nabla \log \pdist[t]{\hat \gamma}(\vect z_t) - \nabla \log \pdist[t]{\gamma^{*}}(\vect z_t)}^2 \mathrm{d}t
    \\
    &\quad\quad\geq\frac{(\hat\gamma-\gamma^*)^2 \epsilon}{16 \hat\gamma (1-\hat\gamma) \gamma^* (1-\gamma^*)} 
    \frac{K\sqrt{K^2 + R_0}+\frac{1}{2}K^2 + \frac{1}{2}R_0}{\sqrt{\ln(\separation)- \ln\fitparenth{K+ \frac{1}{2}\sqrt{K^2 + R_0}}}}
\end{align*}

Assume that $K\geq \tau$, which occurs with probability $1-\Phi(\tau)$.
Our bound is an increasing function in $K$ and is thus minimized at $K=\tau$. Similarly, our upper bound from before on $\epsilon$ is a decreasing function of $K$ and our lower bound from before is an increasing function of $\epsilon$, so our bounds on $\epsilon$ are tightest when we plug in $\tau$, giving $\epsilon_{\mathrm{max}}^0$ and $\epsilon_{\mathrm{min}}^0$ from the statement of the proposition.
\end{proof}

We now state the corresponding proposition for the case where $c=1$. Let $\vect w_t := \vect y_t |c=1$.
\begin{proposition}\label{prop:c1_lowerbound}
    Let $\tau\geq \separation \sqrt{\alpha_1}$, and let $\epsilon>0$ satisfy $ \epsilon \in (\epsilon_{\mathrm{min}}^1, \epsilon_{\mathrm{max}}^1)$ with 
    \begin{align*}
    \epsilon_{\mathrm{max}} ^1& =\exp\fitparenth{-\frac{1}{8}\fitbracket{(\hat r-r^*)^2 + \fitparenth{-(\hat r+r^*)-\tau^2}_+^2}}
    \\
    \epsilon_{\mathrm{min}}^1 &= \exp\fitparenth{-\frac{1}{8}\fitbracket{(\separation(\separation - 2\tau)+(\hat r+r^*))^2 - (\hat r-r^*)^2}}
    \end{align*}
    
    Then, with probability $1-\Phi(\tau)$ over samples of $\vect w_0$, 
    \begin{align*}
        &\int_0^1\norm{\nabla \log \pdist[t]{\hat \gamma}(\vect w_t) - \nabla \log \pdist[t]{\gamma^{*}}(\vect w_t)}^2 \mathrm{d}t
        \\
        &\quad\quad\geq\frac{(\hat \gamma-\gamma^*)^2 \epsilon}{16\sqrt{\beta_1}\; \hat \gamma (1-\hat \gamma) \gamma^* (1-\gamma^*)} 
        \frac{\tau\sqrt{\tau^2 + R_1}+\frac{1}{2}\tau^2 + \frac{1}{2}R_1}{\sqrt{\ln(\separation)- \ln\fitparenth{\tau+ \frac{1}{2}\sqrt{\tau^2 + R_1}}}}
    \end{align*}
    where
    \[R_1:=\sqrt{8\ln \frac{1}{\epsilon}-(\hat r-r^*)^2}+(\hat r+r^*)\]
\end{proposition}

\begin{proof}
    This follows from repeating all of our analysis above for the $c=0$ case with $\gamma^*$ and $\hat \gamma$ replaced with $1-\gamma^*$ and $1-\hat \gamma$, respectively. 
\end{proof}

Combining \cref{prop:c0_lowerbound} and \cref{prop:c1_lowerbound}, we have by our decomposition of $\ldsm$ from the start of this section that 
\[\ldsm (\pdist{\hat \gamma},\pdist{\gamma^*})\geq C (\hat\gamma - \gamma^*)^2\]
where
\[C=\frac{ \epsilon ( 1-\Phi(\tau))}{16\sqrt{\beta_1}\; \hat\gamma (1-\hat\gamma) \gamma^* (1-\gamma^*)}\left[(1-\gamma^*)\frac{\tau\sqrt{\tau^2 + R_0}+\frac{1}{2}\tau^2 + \frac{1}{2}R_0}{\sqrt{\ln(\separation)- \ln\fitparenth{\tau+ \frac{1}{2}\sqrt{\tau^2 + R_0}}}}+\gamma^* \frac{\tau\sqrt{\tau^2 + R_1}+\frac{1}{2}\tau^2 + \frac{1}{2}R_1}{\sqrt{\ln(\separation)- \ln\fitparenth{\tau+ \frac{1}{2}\sqrt{\tau^2 + R_1}}}}\right]\]
for all $\tau\geq \separation\sqrt{\alpha_1}$ and $\epsilon\in [ \epsilon_{\mathrm{min}}, \epsilon_{\mathrm{\max}}]$, where
\begin{align*}
\epsilon_{\mathrm{max}}&=\min\set{\epsilon_{\mathrm{max}}^0,\epsilon_{\mathrm{max}}^1}=\exp\fitparenth{-\frac{1}{8}\fitbracket{(\hat r-r^*)^2 + \max\set{\fitparenth{\hat r+r^*-\tau^2}_+^2,\fitparenth{-(\hat r+r^*)-\tau^2}_+^2}}}
\\
\epsilon_{\mathrm{min}}&=\max\set{\epsilon_{\mathrm{min}}^0,\epsilon_{\mathrm{min}}^1}=\exp\fitparenth{-\frac{1}{8}\fitbracket{\min\set{(\separation(\separation - 2\tau)+(\hat r+r^*))^2,(\separation(\separation - 2\tau)-(\hat r+r^*))^2} + (\hat r-r^*)^2}}
\end{align*}

To simplify this expression, note that if 
\[g(x):=\frac{x(\tau+\frac{1}{2}x)}{\sqrt{\ln(\separation) - \ln(\tau + \frac{1}{2}x)}}\]
then direct calculation shows that if $h(x):=\ln(\separation) - \ln(\tau + \frac{1}{2}x)$,
\[g'(x)=\frac{\frac{x}{4}+(\tau+x)h(x)}{h(x)^{3/2}}>0\]

Therefore, $g$ is increasing, meaning
\begin{align*}
    C&=\frac{ \epsilon ( 1-\Phi(\tau))}{16\sqrt{\beta_1}\; \hat\gamma (1-\hat\gamma) \gamma^* (1-\gamma^*)}\left[(1-\gamma^*)g\fitparenth{\sqrt{\tau^2+R_0}}+\gamma^* g\fitparenth{\sqrt{\tau^2+R_1}}\right]
    \\
    &\geq \frac{ \epsilon ( 1-\Phi(\tau))}{16\sqrt{\beta_1}\; \hat\gamma (1-\hat\gamma) \gamma^* (1-\gamma^*)}g\fitparenth{\min\set{\sqrt{\tau^2+R_0}, \sqrt{\tau^2+R_1}}}
    \\
    &= \frac{ \epsilon ( 1-\Phi(\tau))}{16\sqrt{\beta_1}\; \hat\gamma (1-\hat\gamma) \gamma^* (1-\gamma^*)}g\fitparenth{\sqrt{\tau^2+\sqrt{8\ln \frac{1}{\epsilon}-(\hat r - r^*)^2}- \abs{\hat r-r^*}}}
\end{align*}

Let $R_{\mathrm{min}}:=\sqrt{8\ln \frac{1}{\epsilon}-(\hat r - r^*)^2}- \abs{\hat r-r^*}$. To obtain the bound from the body of the paper, we now specify values for $\tau$ and $\epsilon$ which simplify the expression. First, we will set $\tau$ to be its minimum value: $\tau=\separation\sqrt{\alpha_1}$. We choose this because numerically, it seems to be the value of $\tau$ that maximizes the expression. This gives us
\[C\geq \frac{ \epsilon ( 1-\Phi(\separation\sqrt{\alpha_1}))}{16\sqrt{\beta_1}\; \hat\gamma (1-\hat\gamma) \gamma^* (1-\gamma^*)} \frac{\sqrt{\separation^2\alpha_1 + R_{\mathrm{min}} }\fitparenth{\separation\sqrt{\alpha_1} +\frac{1}{2}\sqrt{\separation^2\alpha_1+ R_{\mathrm{min}}}}}{\sqrt{\log(\separation)-\log\fitparenth{\sqrt{\alpha_1}}-\log\fitparenth{\separation+\frac{1}{2}\sqrt{\separation^2+\frac{R_0}{\alpha_1}}}}}\]

Noting that $\log(\separation)<\log\fitparenth{\separation+\frac{1}{2}\sqrt{\separation^2+\frac{R_0}{\alpha_1}}}$, the denominator of the second term is less than $\sqrt{-\log(\sqrt{\alpha_1})}=\frac{1}{4}\beta_1$ (recall $\alpha_t=\exp\fitparenth{-\int_0^t \beta_s\mathrm{d}s}=\exp\fitparenth{-\frac{1}{2}\beta_1 t^2}$). Additionally, note that the numerator of the second term is an increasing function of $\separation$, meaning we can lower-bound it by setting $\separation=0$. Combining these gives us 
\[C\geq \frac{1-\Phi(\separation\sqrt{\alpha_1})}{8\beta_1\; \hat\gamma (1-\hat\gamma) \gamma^* (1-\gamma^*)} \epsilon \,R_{\mathrm{min}} \]

Finally, direct calculation shows that $\epsilon\, R_{\mathrm{min}}$ is maximized when $\epsilon=\exp\fitparenth{-\frac{1}{8}\fitparenth{(\hat r-r^*)^2 + \delta^2}}$ where $\delta = \frac{|\hat r-r^*|+ \sqrt{(\hat r-r^*)^2+16}}{2}$, giving 
\[\epsilon \, R_{\mathrm{min}} =\frac{\sqrt{(\hat r-r^*)^2+16}-|\hat r-r^*|}{2} \exp\fitparenth{-\frac{1}{8}\fitparenth{(\hat r-r^*)^2 + \delta^2}}\]

Note $\delta\leq \abs{\hat r-r^*}+2$ (because $\sqrt{a^2+b^2}\leq |a|+|b|$), and $\delta^2\leq \abs{\hat r-r^*}^2+2\abs{\hat r-r^*}+4$. Applying this, and the fact that $\hat \gamma(1-\hat \gamma)\leq\frac{1}{4}$
\begin{align*}
C&\geq \frac{\fitparenth{1-\Phi(\separation \sqrt{\alpha_1})}\fitparenth{\sqrt{(\hat r-r^*)^2+16}-|\hat r-r^*|}}{4\beta_1\;  \gamma^* (1-\gamma^*)}\exp\fitparenth{-\frac{1}{4}\fitparenth{(\hat r-r^*)^2 +\abs{\hat r-r^*}} - \frac{1}{2}}
\\
&=\frac{\fitparenth{1-\Phi(\separation \sqrt{\alpha_1})}\fitparenth{\sqrt{(\hat r-r^*)^2+16}-|\hat r-r^*|}}{4\sqrt{e}\beta_1\;  \gamma^* (1-\gamma^*)}\exp\fitparenth{-\frac{1}{4}\fitparenth{(\hat r-r^*)^2 +\abs{\hat r-r^*}}}
\end{align*}

Finally, for $\hat \gamma\in [a,1-a]$, $\abs{\hat r - r^*} \leq \log\frac{1-a}{a}+\abs{\log \frac{1-\gamma^*}{\gamma^*}}$. Our expression is a decreasing function of $\abs{\hat r-r^*}$, and thus for all $\hat \gamma\in [a,1-a]$, if we let $k:=\log\frac{1-a}{a}+\abs{\log \frac{1-\gamma^*}{\gamma^*}}$
\[C\geq \frac{\fitparenth{1-\Phi(\separation \sqrt{\alpha_1})}\fitparenth{\sqrt{k^2+16}-k}}{7\beta_1\;  \gamma^* (1-\gamma^*)}\exp\fitparenth{-\fitparenth{k^2 +k}/4}\]
where we additionally used $4\sqrt{e}<7$. Bringing this all together, 
\begin{align}
   \min_{\hat \gamma \in [a,1-a]}\frac{\ldsm(\pdist{\hat\gamma},\pdist{\gamma^*})}{(\hat\gamma-\gamma^*)^2}>\frac{(1-\Phi(\norm{\mu_1-\mu_0}\sqrt{\alpha_1}))\fitparenth{\sqrt{k^2+16}-k}}{7\beta_1\gamma^*(1-\gamma^*)}\exp\fitparenth{-(k^2+k)/4} 
\end{align}
\qed

\section{Additional numerical experiments}\label{sec:Additional Numerics}

\subsection{Noise schedules and time reparameterization}\label{sec: Reparam}
In this work, we reference two noise schedules used in practice. The linear noise schedule is defined by $\beta_t:=\beta_1 t$ for a constant $\beta_1>0$, such that $\alpha_t=\exp\fitparenth{-\frac{\beta_1}{2}t^2}$. The squared-cosine noise schedule \citep{hang2024improvednoiseschedulediffusion}

$    \bar\alpha_t=\cos\fitparenth{\frac{t+s}{1+s}\cdot \frac{\pi}{2}}^2$ and $\alpha_t =\frac{\bar\alpha_t}{\bar\alpha_0}$
for $s=0.008$.

We want to analyze the effect of artificially reducing the sensitivity index on the estimation of $\gamma^*$. To do this, we can reparameterize the time in the forward process as $t=t_{a,b,v}(\tau)$, where 
\[t_{a,b,v}(\tau)=\begin{cases}
    \tau & 0\leq \tau<a
    \\
    a + v(\tau-a)& a \leq \tau <a + \frac{b-a}{v}
    \\
    b + \frac{1-b}{1-\fitparenth{a + \frac{b-a}{v}}}\fitparenth{t - a - \frac{b-a}{v}} & a + \frac{b-a}{v} \leq t\leq 1
\end{cases}\]

Under this parameterization, if we let $\vect y_\tau:= \vx_{t(\tau)}$ where $\vx_t$ follows the Variance Preserving SDE with noise schedule $\set{\alpha_t}$, then $\vect y_t$ satisfies
\begin{align}
    \mathrm{d}\vect y_\tau &= -\frac{\beta_{t(\tau)}}{2}\vect y_\tau \frac{\mathrm{d}t}{\mathrm{d} \tau}\mathrm{d}\tau + \sqrt{\beta_{t(\tau)}}\cdot \sqrt{\frac{\mathrm{d}t}{\mathrm{d}\tau}}\mathrm dW_t
    \\
    &= -\frac{\tilde{\beta}_{\tau}}{2}\vect y_\tau \mathrm{d}\tau + \sqrt{\tilde{\beta}_{\tau}}\mathrm{d}W_t
\end{align}
where $\tilde{\beta}_\tau:= \beta_{t(\tau)} \cdot \frac{\mathrm{d}t}{\mathrm{d}\tau}$. In particular, reparameterizing time is equivalent to choosing a different noise schedule $\set{\tilde{\alpha}_\tau}$. For $v\geq 1$, $t_{a,b,v}$ maps the interval $[a,b]$ to $\fitbracket{a, a+\frac{b-a}{v}}$, scaling the length of the interval by $1/v$. To reduce the value of $L(\gamma^*)$, we let $[a,b]$ contain the region of time where $\frac{\lsm(\pdist[t]{\gamma}, \pdist[t]{\gamma^*})}{(\gamma-\gamma^*)^2}>\epsilon$ for fixed $\epsilon>0$ and increase the value of $v$.

\subsection{Neural network experiment details}

For the score models $\vs_\btheta$ in our experiments, we use deep ResNets \citep{He2015DeepRL} with width equal to the ambient dimension and depth $8$. Unless otherwise specified, we use a linear noise schedule with $\beta_1=20$ and a time discretization of $T=4000$ steps on the interval $[10^{-5},1]$ for training and $[10^{-3},1]$ for sampling, following \citep{song2021scorebasedgenerativemodelingstochastic}.

We train our models using the loss described in this paper (and \citep{song2021scorebasedgenerativemodelingstochastic}) using AdaBelief \citep{zhuang2020adabelief} with a linear warmup and cosine decay learning rate. We use JAX \citep{jax2018github} and Equinox \citep{kidger2021equinox} for our neural network code, SciPy \citep{2020SciPy-NMeth} for additional machine learning algorithms, Matplotlib \citep{Hunter2007matplotlib} for figure generation, and Weights and Biases \citep{wandb} for experiment tracking.
\footnote{We will make our code available upon publication.}

\subsection{MNIST ablation study details}\label{sec: MNIST Exp Details}

The architecture of our autoencoder is simple: two convolutional layers followed by a dense layer for the encoder, and the reverse for the decoder. For our MNIST experiments, we use a latent dimension of $12$. We train the autoencoder on all examples of $1$s and $8$s in MNIST. We additionally train an MNIST classifier on only $1$s and $8$s using the architecture in the Equinox documentation \citep{kidger2021equinox}.

Unlike in the case of Gaussian mixture models, we do not know the analytical score function for the latents of our autoencoder, and thus estimation of $\ldsm(\pdist[t]{\gamma}, \pdist[t]{\gamma^*})$ and $L(\gamma^*)$ analytically is intractable. Instead, for $\gamma^*\in \set{0.2,0.22,...,0.8}$, we train diffusion models $\vs_{\btheta_\gamma}$ on data from $\pdist{\gamma}$ as approximations for the analytical scores using a linear noise schedule. We keep the number of training examples constant across all experiments (5800, the smaller of the counts of 1s and 8s in MNIST) and train the models identically. For each $t\in [0,1]$, we then have the estimate \[\ldsm(\pdist[t]{\gamma}, \pdist[t]{\gamma^*})\approx \expect[\vx_t\sim \pdist[t]{\gamma^{*}}]{\lambda_t \norm{\vs_{\btheta_{\gamma^*}}(\vx_t, t) - \vs_{\btheta_{\gamma'}}(\vx_t, t)}^2}\]

We generate all samples used in \Cref{fig: MNIST experiment} using $\vs_{\btheta_{0.6}}$, varying the noise schedule used for sampling.




\subsection{DSSI for other parameters }\label{sec: Other Parameters}

Here, we apply our DSSI framework to parameters other than the mixture weight in a Gaussian model. In what follows, we will let $\Sigma_0$ and $\Sigma_1$ be fixed random positive definite matrices, such that 
\[\Sigma_i = A^\top A+10^{-8}I, \;\;A_{ij}\sim \mathcal{N}(0,1)\]
for $i\in \set{0,1}$ and use a linear noise schedule. We focus here on one-dimensional families of distributions for consistency with the body of the paper, although we note that our framework can be extended to multi-parameter recovery directly. For this section, we examine GMMs in dimension $d=10$.

For our analysis of the sensitivity of $\ldsm$ to the modes of the mixture components, we consider the one-dimensional family of distributions 
\[\pdist{\delta}:=0.6 \; \mathcal{N}(\vect 0, \vect \Sigma_0) + 0.4 \; \mathcal{N}\fitparenth{\frac{12\;\delta}{\sqrt{d}}\vect 1,\vect\Sigma_1}\] 
in dimension $d=10$ where we range $\delta$ between $0$ and $1$. \Cref{fig: means heatmap} shows $\lsm(\pdist[t]{\delta}, \pdist[t]{\delta^*})$ as a function of $\delta$ and $t$, and \Cref{fig: means index} shows $\frac{\ldsm(\pdist[t]{\delta}, \pdist[t]{\delta^*})}{(\delta - \delta^*)^2}$ as a function of $\delta$ for multiple values of $\delta^*$. We see that, as with the mixture weight, $\lsm(\pdist[t]{\delta}, \pdist[t]{\delta^*})$ is nonnegligible only for intermediate times in the forward process. However, note that the minimum value of $L(\delta^*)$ we see here is greater than $2$ (compared to $0.37$ in \Cref{fig: gmm score sensitivity}), meaning that our theory implies a $5.4\times$ stronger control on the error in estimation $\mu$ using a diffusion model trained to the same level of error under $\ldsm$.

\begin{figure}
    \centering
    \begin{subfigure}[htbp]{0.45\linewidth}
        \includegraphics[width=\columnwidth]{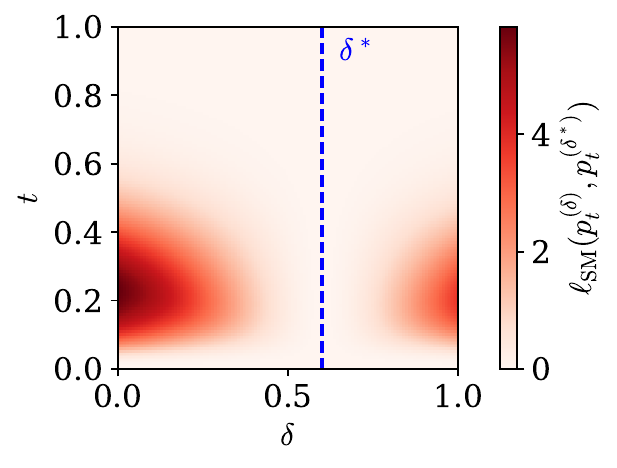}
        \caption{}
        \label{fig: means heatmap}
    \end{subfigure}
    \hfill
    \begin{subfigure}[htbp]{0.45\linewidth}
        \includegraphics[width=\columnwidth]{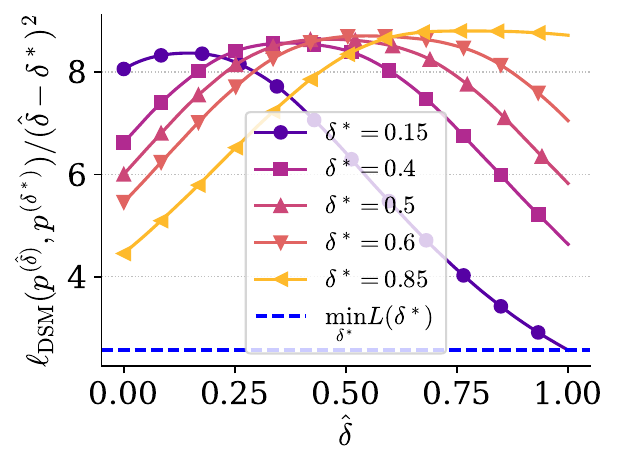}
        \caption{}
        \label{fig: means index}
    \end{subfigure}
    \caption{Sensitivity of $\ldsm$ to means of mixture components. (a) $\lsm(\pdist[t]{\delta}, \pdist[t]{\delta^*})$ as a function of $t$ and $\delta$. (b)  $\frac{\ldsm(\pdist{\delta}, \pdist{\delta^*})}{(\delta - \delta^*)^2}$ as a function of $\delta$.}
\end{figure}

We now apply our DSSI framework to covariance matrix recovery. Consider the one-dimensional family of distributions 
\[\pdist{\kappa}:=0.6 \; \mathcal{N}(\vect 0, \vect I) + 0.4 \; \mathcal{N}\fitparenth{\frac{9.0}{\sqrt{d}}\vect 1,\vect\Sigma_\kappa}\] 
in dimension $d=10$ where we range $\delta$ between $0$ and $1$. Here, we define 
\[\vect \Sigma_\kappa:= \vect \Sigma_0^{1/2}\fitparenth{\vect \Sigma_0^{-1/2}\vect \Sigma_1\vect \Sigma_0^{-1/2}}^{\kappa}\vect \Sigma_0^{1/2}\]
to be the geodesic interpolating between $\vect \Sigma_0$ and $\vect \Sigma_1$.  \Cref{fig: variance heatmap} shows $\lsm(\pdist[t]{\kappa}, \pdist[t]{\kappa^*})$ as a function of $\kappa$ and $t$, and \Cref{fig: variance index} shows $\frac{\ldsm(\pdist{\kappa}, \pdist{\kappa^*})}{(\kappa - \kappa^*)^2}$ as a function of $\kappa$.

Again, $\lsm(\pdist[t]{\delta}, \pdist[t]{\delta^*})$ is nonnegligable for intermediate times in the forward process. We additionally see that $L(\kappa^*)$ is greater than $0.6$ uniformly across $\kappa$.

\begin{figure}
    \centering
    \begin{subfigure}[htbp]{0.45\linewidth}
        \includegraphics[width=\columnwidth]{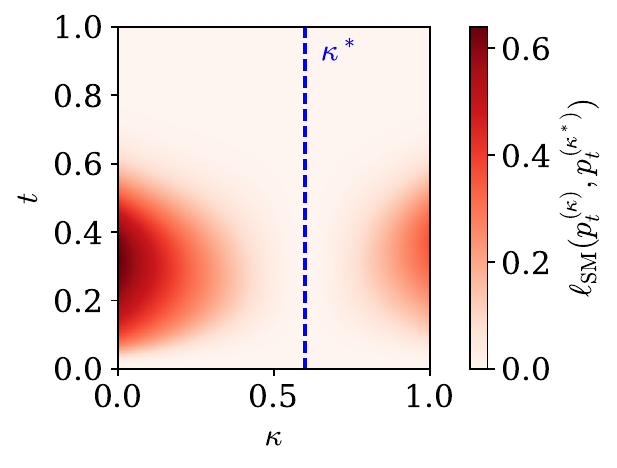}
        \caption{}
        \label{fig: variance heatmap}
    \end{subfigure}
    \hfill
    \begin{subfigure}[htbp]{0.45\linewidth}
        \includegraphics[width=\columnwidth]{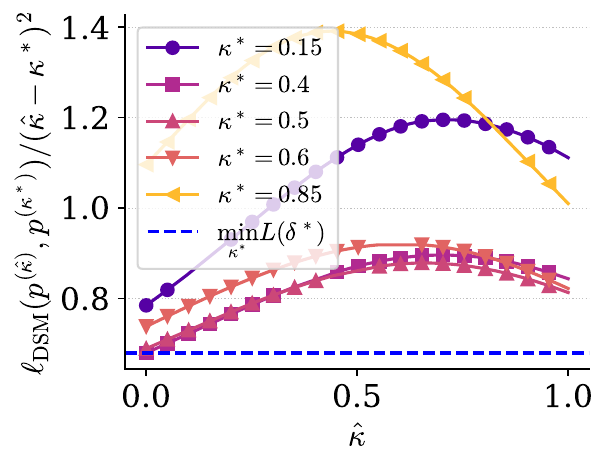}
        \caption{}
        \label{fig: variance index}
    \end{subfigure}
    \caption{Sensitivity of $\ldsm$ to covariance matrix of mixture components. (a) $\lsm(\pdist[t]{\kappa}, \pdist[t]{\kappa^*})$ as a function of $t$ and $\kappa$. (b)  $\frac{\ldsm(\pdist{\kappa}, \pdist{\kappa^*})}{(\kappa - \kappa^*)^2}$ as a function of $\kappa$.}
\end{figure}

\end{appendices}
\end{document}